\documentclass[twoside,11pt]{article}

\usepackage{blindtext}
\usepackage{multirow}
 \usepackage[abbrvbib, preprint]{jmlr2e}

\usepackage{amsmath}
\usepackage{amsfonts}%
\usepackage{amssymb}

\usepackage{graphicx}
\usepackage{enumerate}
\usepackage{natbib}
\usepackage{url} % not crucial - just used below for the URL 
\usepackage{xr}

\usepackage{subcaption, caption}

\usepackage{bbm}
\usepackage{bm}
\usepackage{enumerate}
\usepackage{comment}
\usepackage{subcaption}
\usepackage{xcolor}
\usepackage{comment}

\usepackage[ruled,linesnumbered]{algorithm2e}

\newcommand{\ba}{ {\boldsymbol a} }
\newcommand{\bA}{ {\boldsymbol A} }
\newcommand{\bb}{ {\boldsymbol b} }
\newcommand{\bB}{ {\boldsymbol B} }

\newcommand{\bI}{ {\boldsymbol I} }

\newcommand{\bP}{ {\boldsymbol P} }

\newcommand{\bs}{ {\boldsymbol s} }

\newcommand{\bu}{ {\boldsymbol u} }

\newcommand{\bv}{ {\boldsymbol v} }
\newcommand{\bV}{ {\boldsymbol V} }
\newcommand{\bw}{ {\boldsymbol w} }

\newcommand{\bx}{ {\boldsymbol x} }
\newcommand{\bX}{ {\boldsymbol X} }
\newcommand{\by}{ {\boldsymbol y} }

\newcommand{\balpha}  { {\boldsymbol \alpha} }
\newcommand{\bbeta}   { {\boldsymbol \beta} }
\newcommand{\bgamma}  { {\boldsymbol \gamma} }

\newcommand{\bDelta}  { {\boldsymbol \Delta} }
\newcommand{\bepsilon}{ {\boldsymbol \epsilon} }

\newcommand{\bet}     { {\boldsymbol \eta} }

\newcommand{\bmu}     { {\boldsymbol \mu} }

\newcommand{\bPhi}    { {\boldsymbol \Phi} }

\newcommand{\bpsi}    { {\boldsymbol \psi} }

\newcommand{\bOmega}  { {\boldsymbol \Omega} }

\def\T{{ \mathrm{\scriptscriptstyle T} }}

\newcommand{\given}{\, |\,}

\usepackage{lastpage}
\jmlrheading{26}{2025}{1-\pageref{LastPage}}{4/23; Revised
8/24}{5/25}{23-0505}{Rajarshi Guhaniyogi, Laura Baracaldo and Sudipto Banerjee}
\ShortHeadings{Bayesian Data Sketching}{Guhaniyogi, Baracaldo and Banerjee}

\begin{document}

\title{Bayesian Data Sketching for Varying Coefficient Regression Models}

\author{\name Rajarshi Guhaniyogi \email rajguhaniyogi@tamu.edu \\
       \addr Department of Statistics\\
       Texas A \& M University\\
       College Station, TX 77843-3143, USA
       \AND
       \name Laura Baracaldo \email lnbaracaldol@ucsb.edu  \\
       \addr Department of Statistics and Applied Probability\\
       University of California Santa Barbara\\
       Santa Barbara, CA 93106-3110, USA
       \AND
       \name Sudipto Banerjee \email sudipto@ucla.edu \\
       \addr UCLA Department of Biostatistics \\
       University of California Los Angeles \\
       Los Angeles, CA 90095-1772, USA.}

\editor{Ryan Adams}

\maketitle

\begin{abstract}
Varying coefficient models are popular for estimating nonlinear regression functions in functional data models. Their Bayesian variants have received limited attention in large data applications, primarily due to prohibitively slow posterior computations using Markov chain Monte Carlo (MCMC) algorithms. We introduce Bayesian data sketching for varying coefficient models to obviate computational challenges presented by large sample sizes. To address the challenges of analyzing large data, we compress the functional response vector and predictor matrix by a random linear transformation to achieve dimension reduction and conduct inference on the compressed data. Our approach distinguishes itself from several existing methods for analyzing large functional data in that it requires neither the development of new models or algorithms, nor any specialized computational hardware while delivering fully model-based Bayesian inference. Well-established methods and algorithms for varying coefficient regression models can be applied to the compressed data. %Apart from offering dimension reduction, we also achieve data privacy since the original data are linearly transformed with a random compression matrix and access to spatial locations need not be accessed in the course of the statistical analysis. We call our approach as ``geostatistical sketching" (geoS) and explore performance of geoS with different options of modelling for varying coefficients in the spatial VCMs. In contrast to a gamut of literature on data sketching in high dimensional ridge regression, hitherto little is known about the large sample properties for Bayesian SVCMs under data sketching. In this article, 
We establish posterior contraction rates for estimating the varying coefficients and predicting the outcome at new locations with the randomly compressed data model. We use simulation experiments and analyze remote sensed vegetation data to empirically illustrate the inferential and computational efficiency of our approach.

\end{abstract}

\begin{keywords}
B-splines, Predictive Process, Posterior contraction, Random compression matrix, Varying coefficient models.
\end{keywords}

\section{Introduction}
\label{sec:intro}

%Body of paper.  Margins in this document are roughly 0.75 inches all around, letter size paper.

We develop a statistical learning framework for functional data analysis using Bayesian data sketching to deliver inference that scales massive functional datasets. %We introduce a different approach for analysing large spatial data sets using (\ref{eq: svc_basic_general}) that also offers data privacy. Our approach builds upon the recent literature on 
``Data sketching'' \citep{vempala2005random, halko2011finding, mahoney2011randomized, woodruff2014sketching, guhaniyogi2015bayesian, guhaniyogi2016compressed} is a compression method that is increasingly used to analyze massive amounts of data. The entire data set is compressed before being analyzed for computational efficiency. Data sketching proceeds by transforming the original data through a random linear transformation to produce a much smaller number of data samples. We analyze the compressed data, thereby achieving dimension reduction. %Furthermore, the original data is neither accessed nor exactly recoverable from the compressed data, which preserves data confidentiality. The random linear transformation compresses the data, which imparts scalability because the model is run on the compressed data. 

Such developments have focused mainly on ordinary linear regression and penalized linear regression \citep{zhang2013recovering, chen2015fast, dobriban2018new, drineas2011faster, ahfock2017statistical, huang2018near}, we develop such methods for functional regression models. Our primary challenge is probabilistic learning for the underlying effects of functional coefficients in the context of varying regression models. Although we have some similarities, our current contribution differs from compressed sensing \citep{donoho2006compressed, ji2008bayesian, candes2006near, eldar2012compressed, yuan2014tree} in inferential objectives. Specifically, compressed sensing solves an inverse problem by ``nearly'' recovering a sparse vector of responses from a smaller set of random linear transformations. In contrast, our functionally indexed response vector is not necessarily sparse. Also, we do not seek to recover (approximately) the original values in the response vector.

We consider a varying-coefficient model (VCM) where all functional variables (response and predictors) are defined in a $d$-dimensional indexed space $\mathcal{D}\subseteq\mathbb{R}^d$. For temporal data $d=1$ and for spatial data applications $d=2$, while for spatial-temporal applications the domain is ${\cal D} = \mathbb{R}^{2}\times \mathbb{R}^{+}$ and the index is a space-time tuple $(\bu = (\bs, t))$. For each index $\bu\in\mathcal{D}$, the functional response $y(\bu)\in\mathcal{Y}\subseteq\mathbb{R}$ and $P$ functional predictors $x_1(\bu),...,x_P(\bu)\in\mathcal{X}\subseteq\mathbb{R}$, are related according to a posited varying coefficients regression model
\begin{align}\label{eq: svc_basic_general}
y(\bu) = \sum_{j=1}^P x_j(\bu)\beta_j + \sum_{j=1}^{\tilde{P}} \tilde{x}_j(\bu)w_j(\bu) + \epsilon(\bu) = \bx(\bu)^{\T}\bbeta + \tilde{\bx}(\bu)^{\T}\bw(\bu) + \epsilon(\bu)\;,
\end{align}
where $\bbeta=(\beta_1,\beta_2,\ldots,\beta_P)^{\T}$ is a $P\times 1$ vector of functionally static coefficients, $\tilde{\bx}(\bu) = (\tilde{x}_1(\bu),\tilde{x}_2(\bu),\ldots,  \tilde{x}_{\tilde{P}}(\bu))^{\T}$ is a $\tilde{P}\times 1$ vector comprising a subset of predictors from $\bx(\bu)$ (so $\tilde{P} \leq P$) whose impact on the response is expected to vary over the functional inputs, $\bw(\bu)=(w_1(\bu), w_2(\bu),\ldots, w_{\tilde{P}}(\bu))^{\T}$ is a $\tilde{P}\times 1$ vector of functionally varying regression slopes, and $\epsilon(\bu) \stackrel{iid}{\sim} N(0,\sigma^2)$ captures measurement error variation at location $\bu$. Functionally varying regression coefficient models are effective in learning the varying impact of predictors on the response in time series \citep[see, e.g.,][and references therein]{chen1993functional,cai2000functional} and in spatial applications
\citep[see, e.g.,][and references therein]{gelfand2003spatial, banerjeejohnson2006, wheeler2007assessment, finley2011hierarchical, guhaniyogi2013modeling, finleyBanerjee2020, kim_wang} and in spatial-temporal data analysis \citep[see, e.g.,][and references therein]{lee2021clustered}. When $d=2$, customary geostatistical regression models with only a spatially-varying intercept emerge if the first column of $x(\bu)$ is the intercept and $\tilde{P}=1$ with $\tilde{x}_1(\bu) = 1$. Spatially varying coefficient models, a class of varying coefficient models for $d=2$, also offer a process-based alternative to the widely used geographically weighted regression \citep[see, e.g.,][]{brundson_gwr} to model non-stationary behavior in the mean. \cite{finley2011gwr} offers a comparative analysis and highlights the richness of (\ref{eq: svc_basic_general}) in ecological applications.

Bayesian inference for (\ref{eq: svc_basic_general}) is computationally expensive for large data sets due to the high-dimensional covariance matrix introduced by $w(\bu)$ in (\ref{eq: svc_basic_general}). The modeling of high-dimensional dependent functional data has been attracting significant interest, and the growing literature on scalable methods, which has been adapted and built on scalable spatial models \citep[see, e.g.,][for reviews in spatial statistics]{banerjee2017high, heaton2019case}, is too vast to be comprehensively reviewed here. Briefly, model-based dimension reduction in functional data models have proceeded from fixed-rank representations \citep[e.g.,][]{cressie2008fixed, banerjee2008gaussian, wikle_2011, snelson2005sparse, burt2020convergence}, multi-resolution approaches \citep[e.g.,][]{nychka2015,guhaniyogi2018large}, sparsity-inducing processes \citep[e.g.,][]{ve88, datta16, zhangDattaBanerjee2019, katzfuss2021general, pbf2021} and divide-and-conquer approaches such as meta-kriging \citep{guhaniyogi2018meta, guhaniyogi2017divide,guhaniyogi2020distributed}. 

While most of the aforementioned methods entail new classes of models and approximations, or very specialized high-performance computing architectures, Bayesian data sketching has the advantage that customary exploratory data analysis tools, well-established methods, and well-tested available algorithms for implementing (\ref{eq: svc_basic_general}) can be applied to the sketched data without requiring new algorithmic or software development. We pursue fully model-based Bayesian data sketching, where inference proceeds from a hierarchical model \citep{cressie2015statistics, ban14}. The hierarchical approach to functional data analysis is widely employed for inferring on model parameters that may be weakly identified from the likelihood alone and, more relevantly for substantive inference, for estimating the functional relationship between response and predictors over the domain of interest. For analytic tractability we model the varying coefficients using basis expansions \citep{wikle_2011, wang2008variable, wang2009shrinkage, bai2019fast} rather than Gaussian processes. 

We exploit some recent developments in the theory of random matrices to relate the inference from the compressed data with the full-scale functional data model. We establish consistency of the posterior distributions of the varying coefficients and analyze the predictive efficiency of our models based on the compressed data. Posterior contraction of varying-coefficient (VC) models have been investigated by a few recent articles. For example, \cite{guhaniyogi2020distributed} derive minimax optimal posterior contraction rates for Bayesian VC models under GP priors when the number of predictors $P$ is fixed. \cite{deshpande2020vcbart} also derived near-optimal posterior contraction rates under BART priors, and \cite{bai2019fast} provided an asymptotically optimal rate of estimation for varying coefficients with a variable selection prior on varying coefficients. We address these questions in the context of data compression, which has largely remained unexplored.

While our approach randomly compresses the data for efficient Bayesian inference, there is a related but distinct approach that relies on stochastic gradient decent with subsampled or mini-batch input at each iteration for efficient computation. Traditionally, both proposal generation and acceptance test within the Metropolis-Hastings algorithm require a full pass over the data, which results in reduced efficiency. Subsampling approaches address this by using mini-batches or subsets of data for both the proposal step and the acceptance test. For the proposal, \cite{welling2011bayesian} introduced the Stochastic Gradient Langevin Dynamics (SGLD) algorithm, a variant of the first-order Langevin dynamics that adds noise to ensure the correct noise distribution. They also anneal the step size to zero, eliminating the need for an acceptance test. \cite{ahn2012bayesian} moved away from Langevin dynamics and proposed a method based on Fisher scoring. \cite{chen2014stochastic} introduced the Stochastic Gradient Hamiltonian Monte Carlo (SGHMC) algorithm, which is based on a variant of second-order Langevin dynamics with momentum to update the state. Similarly to SGLD, SGHMC injects additional noise, but also reduces the effect of gradient noise. For the acceptance test, \cite{korattikara2014austerity} proposed a sequential hypothesis test for Metropolis-Hastings proposals based on a fraction of the full dataset. Building on this seminal work, other mini-batch MH algorithms were developed by \cite{seita2016efficient} and \cite{bardenet2014towards}. In recent years, mini-batched approaches in Markov chain Monte Carlo (MCMC) have expanded to include tempered methods \citep{li2017mini}, Gibbs sampling \citep{de2018minibatch}, and gradient-based proposals \citep{wu2022mini}, with a comprehensive review provided by \cite{bardenet2017markov}. Subsampling-based approaches are further extended to derive distributed Bayesian approaches, where instead of computing a gradient with different subsamples at each step, posterior distributions are independently fitted on different mini-batches followed by combining the inferences from these mini-batches \citep{guhaniyogi2018meta,guhaniyogi2019multivariate,guhaniyogi2020distributed,guhaniyogi2023distributed}. Here, we build the compressed model and are agnostic to the specific estimation algorithm. In fact, while we use MCMC in subsequent analysis, alternative approaches such as predictive stacking \citep{zhangTangBanerjee2024} can be used to learn about the functional coefficients.  

The balance of this article proceeds as follows. Section~\ref{sec:geostatistical} develops our data sketching approach and discusses Bayesian implementation of VC models with sketched data. Section~\ref{sec:theory} establishes posterior contraction rates for varying coefficients under data sketching. Section~\ref{sec:empirical} demonstrates performance of the proposed approach with simulation examples and a forestry data analysis. Finally, Section~\ref{sec: summary} concludes the paper with an eye toward future extensions. All proofs of the theoretical results are placed in the Appendix.

%The rest of the article proceeds as following. Section~\ref{sec:geostatistical} outlines the data sketching approach and discusses Bayesian implementation of SVCMs with sketched data. Section~\ref{sec:theory} derives posterior contraction rate results for varying coefficients under data sketching. Section~\ref{sec:empirical} demonstrates performance of the proposed approach with simulation examples and a forestry data analysis. Finally, Section~\ref{sec:conclusion} concludes the paper with an eye towards future extensions. All proofs of the theoretical results are placed in the Appendix.

\section{Bayesian Compressed Varying Coefficient Models}\label{sec:geostatistical}
{
\subsection{Model construction}\label{sec:model_construction}
}
We model each varying coefficient $w_j(\bu)$ in (\ref{eq: svc_basic_general}) as
\begin{align}\label{basis}
w_j(\bu) = \sum_{h=1}^H B_{jh}(\bu)\gamma_{jh}\;,\quad j=1,...,\tilde{P}\;, 
\end{align}
where each $B_{jh}(\bu)$ is a basis function evaluated at an index $\bu$ for $h=1,...,H$, and $\gamma_{jh}$'s are the corresponding basis coefficients. The distribution of these $\gamma_{jh}$'s  yields a multivariate process with $\mbox{cov}(w_i(\bu),w_j(\bu')) = \bB_i(\bu)^{\T}\mbox{cov}(\bgamma_i,\bgamma_j)\bB_j(\bu)$, where $\bB_i(\bu)$ and $\bgamma_i$ are $H\times 1$ with elements $B_{ih}(\bu)$ and $\gamma_{ih}$, respectively, for $h=1,\ldots,H$. 

Appropriate basis functions can produce appropriate classes of multivariate functional processes. Several choices are available. For example, \cite{biller2001bayesian} and \cite{huang2015bayesian} use splines to model the $B_{jh}(\bu)$'s and place Gaussian priors on the basis coefficients $\gamma_{jh}$. \cite{li2015bayesian} propose a scale mixture of multivariate normal distributions to shrink groups of basis coefficients toward zero. More recently, \cite{bai2019fast} proposed using B-spline basis functions and multivariate spike-and-slab discrete mixture prior distributions on basis coefficients to aid selection of functional variables. Other popular choices for basis functions include wavelet basis \citep{vidakovic2009statistical, cressie2015statistics}, radial basis \citep{bliznyuk2008bayesian}, and locally bisquare \citep{cressie2008fixed} or elliptical basis functions \citep{lemos2009spatio}. Alternatively, a basis representation of $w_j(\bu)$ can be constructed by envisioning $w_j(\bu)$ as the projection of a Gaussian process $w_j(\bu)$ onto a set of reference points, or ``knots'', producing predictive processes or sparse Gaussian processes \citep[see, e.g.,]{snelson2005sparse, banerjee2008gaussian, guhaniyogi2013modeling}. More generally, each $w_j(\bu)$ can be modeled using multiresolution analogues of the aforementioned models to capture global variations at lower resolution and local variations at higher resolutions \citep{katzfussmultires, guhaniyogi2018large}.

Let $\{y(\bu_i), \bx(\bu_i)\}$ be observations at the $N$ index points ${\cal U} = \{\bu_1, \bu_2,\ldots, \bu_N\}$. Using (\ref{basis}) in (\ref{eq: svc_basic_general}) yields the Gaussian linear mixed model
\begin{align}\label{eq: stack_model}
\by = \bX\bbeta + \tilde{\bX}\bB\bgamma + \bepsilon\;,\quad \bepsilon\sim N(0,\sigma^2I_N)\;.
\end{align}
where $\by = (y(\bu_1), y(\bu_2),\ldots, y(\bu_N))^{\T}$ and $\bepsilon = (\epsilon(\bu_1), \epsilon(\bu_2),\ldots, \epsilon(\bu_N))^{\T}$ are $N\times 1$ vectors of responses and errors, respectively, $\bX$ is $N\times P$ with $n$-th row $\bx(\bu_n)^{\T}$, $\tilde{\bX}$ is the $N\times N\tilde{P}$ block-diagonal matrix with $(n,n)$-th block $\tilde{\bx}(\bu_n)^{\T}$, $\bB = (\bB(\bu_1)^{\T},\ldots, \bB(\bu_N)^{\T})^{\T}$ is $N\tilde{P}\times H\tilde{P}$ with $\bB(\bu_n)$ a block-diagonal $\tilde{P}\times H\tilde{P}$ matrix whose $j$-th diagonal block is $(B_{j1}(\bu_n),\ldots,B_{jH}(\bu_n))$. The coefficient $\bgamma = (\bgamma_1^{\T},...,\bgamma_{\tilde{P}}^{\T})^{\T}$ is $H\tilde{P}\times 1$ with each $\bgamma_j=(\gamma_{j1},\ldots,\gamma_{jH})^{\T}$ being $H\times 1$. Bayesian methods for estimating (\ref{eq: stack_model}) typically employ a multivariate normal prior \citep{biller2001bayesian, huang2015bayesian} or its scale-mixture (discrete as well as continuous) variants \citep{li2015bayesian, bai2019fast} on $\bgamma$.

While the basis functions project the coefficients into a low-dimensional space, working with (\ref{eq: stack_model}) will be still be expensive for large $N$ and will be impracticable for delivering full inference (with robust probabilistic uncertainty quantification) for data sets with $N \sim 10^5+$ on modest computing environments. Furthermore, as is well understood in linear regression, specifying a small number of basis functions in (\ref{eq: stack_model}) can lead to substantial over-smoothing and, consequently, biased residual variance estimates in functional varying coefficient models \citep[see, e.g., the discussion in Section~2.1 of][ including Figures~1~and~2 in the paper]{banerjee2017high}.  Instead, we consider data compression or sketching using a random linear mapping to reduce the size of the data from $N$ to $M$ observations. For this, we use $M$ one-dimensional linear mappings of the data encoded by an $M\times N$ compression matrix $\bPhi$ with $M<<N$. This compression matrix is applied to $\by$, $\bX$ and $\tilde{\bX}$ to construct the $M\times 1$ compressed response vector $\by_{\bPhi} = \bPhi \by$ and the matrices $\bX_{\bPhi} = \bPhi \bX$ and $\tilde{\bX}_{\bPhi} = \bPhi \tilde{\bX}$. We will return to the specification of $\bPhi$, which, of course, will be crucial for relating the inference from the compressed data with the full model. For now assuming that we have fixed $\bPhi$, we construct a Bayesian hierarchical model for the compressed data
\begin{equation}\label{eq: svc_basic_compressed_bhm}
 p(\bpsi, \bbeta, \bgamma, \sigma^2 \given \by_{\bPhi}, \bPhi) \propto p(\bpsi, \sigma^2, \bbeta, \bgamma) \times N(\by_{\bPhi}\given \bX_{\bPhi}\bbeta + \tilde{\bX}_{\bPhi}\bB\bgamma, \sigma^2I_M)\;,
\end{equation}
where $\bpsi$ denotes additional parameters specifying the prior distributions on either $\bgamma$ or $\bbeta$. For example, a customary specification is 
\begin{equation}\label{eq: svc_basic_compressed_bhm_priors}
    p(\bpsi, \sigma^2, \bbeta, \bgamma) = \prod_{i=1}^{\tilde{P}}IG(\tau_i^2\given a_{\tau},b_{\tau}) \times IG(\sigma^2\given a_{\sigma},b_{\sigma})\times N(\bbeta\given \bmu_{\beta}, \bV_{\beta}) \times N(\bgamma \given {\boldsymbol 0}, \bDelta)\;,
\end{equation}
where $\bpsi=\{\tau_1^2,...,\tau_{\tilde{P}}^2\}$ and $\bDelta$ is $H\tilde{P}\times H\tilde{P}$ block-diagonal with $j$-th block given by $\tau_j^2\bI_H$, for $j=1,...,\tilde{P}$. While (\ref{eq: svc_basic_compressed_bhm_priors}) is a convenient choice for empirical investigations due to conjugate full conditional distributions, our method applies broadly to any basis function and any discrete or continuous mixture of Gaussian priors on the basis coefficients. In applications where the associations among the latent regression slopes is of importance, one could, for instance, adopt $p(\bpsi,\bgamma) = IW(\bpsi\given r, \bOmega)\times N(\bgamma\given 0,\bDelta_\bpsi)$ with $\bpsi$ as the $H\tilde{P}\times H\tilde{P}$ covariance matrix for $\bgamma$. Our current focus is not, however, on such multivariate models, so we do not discuss them further except to note that (\ref{eq: svc_basic_compressed_bhm}) accommodates such extensions. 

The likelihood in (\ref{eq: svc_basic_compressed_bhm}) is different from that by applying $\bPhi$ to (\ref{eq: stack_model}) because the error distribution in (\ref{eq: svc_basic_compressed_bhm}) is retained as the usual noise distribution without any effect of $\bPhi$. Hence, the model in (\ref{eq: svc_basic_compressed_bhm}) is a model analogous to (\ref{eq: stack_model}) but applied to the \emph{new} compressed data $\{\by_{\bPhi}, \bX_{\bPhi}, \tilde{\bX}_{\bPhi}\}$. However, (\ref{eq: svc_basic_compressed_bhm}) can be regarded as an approximately compressed version of (\ref{eq: stack_model}) when $\bPhi$ is a random matrix constructed in a manner customary for sketching matrices \citep{sarlos2006improved}. To see this, note that a compressed version of Equation~(\ref{eq: stack_model}) will lead to an error $\bPhi\bepsilon$ which follows N($\boldsymbol{0},\sigma^2\bPhi\bPhi^T)$. Lemma 5.36 and Remark 5.40 of \cite{vershynin2010introduction} ensure that when $\bPhi$ is a random matrix constructed as described in this article, the condition $||\bPhi\bPhi^T-\bI||\leq C'\sqrt{M/N}$ for some constant $C'$ is met with probability at least $1-\exp(-C''M)$.
%When $\bPhi$ is a random matrix constructed following the strategy in this article, by Lemma 5.36 and Remark 5.40 of \cite{vershynin2010introduction}, $||\bPhi\bPhi^T-\bI||\leq C'\sqrt{M/N}$, for some constant $C'$, with a at least $1-\exp(-C''M)$ probability. 
Hence, with a very high probability, $\bPhi\bepsilon$ behaves approximately as an $M$-dimensional i.i.d. noise when $M/N$ is small, thus building a connection between Model~(\ref{eq: stack_model})~and~(\ref{eq: svc_basic_compressed_bhm}).
This connection is also key to the computational benefits offered by our model, since working with a $\bPhi$-transformed model (\ref{eq: stack_model}), where the noise distribution is transformed according to $\bPhi\bepsilon$, will not deliver the computational benefits we desire.%, and is somewhat detrimental to the cause of data confidentiality (as in that case, the analyst need to know $\bPhi$) that are provided by (\ref{eq: svc_basic_compressed_bhm})

For specifying $\bPhi$ we pursue ``data oblivious Gaussian sketching'' \citep{sarlos2006improved}, where we draw the elements of $\bPhi = (\Phi_{ij})$ independently from N($0,1/N$) and fix them. The dominant computational operations for obtaining the sketched data using Gaussian sketches is $O(MN^2\tilde{P})$. While Gaussian sketching constructs dense matrices, there are alternative options, oblivious to data, such as the Hadamard sketch \citep{ailon2009fast} and the Clarkson - Woodruff sketch \citep{clarkson2017low} that are available for $\bPhi$. These strategies employ discrete distributions (e.g., Rademacher distribution), instead of a Gaussian distribution, to construct sparse random matrices, which enhances computational efficiency for sketching large data matrices. However, this is less crucial in Bayesian settings since the computation time of (\ref{eq: svc_basic_compressed_bhm}) far exceeds that for the construction of the sketched data. The compressed data serves as a surrogate for the Bayesian regression analysis with varying coefficients. Since the number of compressed records is much smaller than the number of records in the uncompressed data matrix, model fitting becomes computationally efficient and economical in terms of storage as well as the number of floating point operations (flops). Importantly, the original data are not recoverable from the compressed data, and the latter reveal no more information than would be revealed by a completely new sample \citep{zhou2008compressed}. In fact, the original uncompressed data does not need to be stored or accessed at any stage in the course of the analysis.

%While our empirical investigation is limited to this prior in Section~\ref{sec:empirical}, the core idea behind the development applies broadly to any basis function and any discrete or continuous mixture of Gaussian priors on the basis coefficients $\gamma$. Finally, the prior specification is completed by setting $\tau_1^2,...,\tau_{\tilde{P}}^2\stackrel{i.i.d.}{\sim}IG(a_{\tau},b_{\tau})$, $\beta\sim N(0,I_{P})$ and $\sigma^2\sim IG(a_{\sigma},b_{\sigma})$. We set hyper-parameters $a_{\tau}=a_{\sigma}=b_{\tau}=b_{\sigma}=1$.
%With prior distribution on $\gamma$ set as a Gaussian, posterior computation requires drawing Markov chain Monte Carlo (MCMC) samples sequentially from the full conditional posterior distributions of $\gamma|-$, $\Delta|-$ and $\sigma^2|-$. With $\sigma^2\sim IG(a_{\sigma},b_{\sigma})$ distribution a priori, $\sigma^2|-\sim IG(a_{\sigma}+M/2,b_{\sigma}+||\tilde{y}-\tilde{X}_DB\gamma||^2/2)$ which is computationally simple. Similarly, assuming $\Delta$ to be a block diagonal matrix $\Delta=Block-diag(\tau_1^2 I_H,...,\tau_P^2I_H)$

\subsection{Posterior Computations \& Predictive Inference}\label{sec:eff_post}
In what follows, we discuss efficient computation offered by the data sketching framework. With prior distributions on parameters specified as in (\ref{eq: svc_basic_compressed_bhm_priors}), posterior computation requires drawing Markov chain Monte Carlo (MCMC) samples sequentially from the full conditional posterior distributions of $\bgamma|-$, $\bbeta|-$, $\sigma^2|-$ and $\tau_j^2|-$, $j=1,\ldots,\tilde{P}$. To this end, $\sigma^2|-\sim IG(a_{\sigma}+M/2,b_{\sigma}+||\by_{\bPhi}-\bX_{\bPhi}\bbeta-\tilde{\bX}_{\bPhi}\bB\bgamma||^2/2)$, $\tau_j^2|-\sim IG(a_{\tau}+H/2,b_{\tau}+||\bgamma_j||^2/2)$ and $\bbeta|-\sim N\left(\left(\bX_{\bPhi}^{\T}\bX_{\bPhi}/\sigma^2+\bI\right)^{-1}\bX_{\bPhi}^{\T}(\by_{\bPhi}-\tilde{\bX}_{\bPhi}\bB\bgamma)/\sigma^2,\left(\bX_{\bPhi}^{\T}\bX_{\bPhi}/\sigma^2+\bI\right)^{-1}\right)$ do not present any computational obstacles. The main computational bottleneck lies with $\bgamma|-$,
\begin{align}\label{cond_beta}
N\left(\left(\frac{\bB^{\T}\tilde{\bX}_{\bPhi}^{\T}\tilde{\bX}_{\bPhi}\bB}{\sigma^2} +\bDelta^{-1}\right)^{-1}\bB^{\T}\tilde{\bX}_{\bPhi}^{\T}\frac{(\by_{\bPhi}-\bX_{\bPhi}\bbeta)}{\sigma^2},(\bB^{\T}\tilde{\bX}_{\bPhi}^{\T}\tilde{\bX}_{\bPhi}\bB/\sigma^2+\bDelta^{-1})^{-1}\right).
\end{align}
Efficient sampling of $\bgamma$ uses the Cholesky decomposition of $\left(\bB^{\T}\tilde{\bX}_{\bPhi}^{\T}\tilde{\bX}_{\bPhi} \bB/\sigma^2+\bDelta^{-1}\right)$ and solves triangular linear systems to draw a sample from (\ref{cond_beta}). While numerically robust for small to moderately large $H$, computing and storing the Cholesky factor involves $O((H\tilde{P})^3)$ and $O((H\tilde{P})^2)$ floating point operations, respectively \citep{golub2012matrix}. This produces bottlenecks for a large number of basis functions, which is required to estimate the functional coefficients with sufficient local variation. 

To achieve computational efficiency, we adapt a recent algorithm proposed in \cite{bhattacharya2016fast} (in the context of ordinary linear regression with uncompressed data and small sample size) to our setting, with the details provided in Algorithm~\ref{alg1}. Predictive inference on $y(\bu_0)$ proceeds from the posterior predictive distribution %$\mathbb{E}[p(y(s_0)\given y_{\Phi}, \beta, \gamma, \sigma^2)]$, 
\begin{equation}\label{eq: posterior_predictive_compressed_bhm}
\mathbb{E}[p(y(\bu_0)\given \by_{\bPhi}, \bbeta, \bgamma, \sigma^2)] = \int p(y(\bu_0)\given \by_{\bPhi}, \bbeta, \bgamma, \sigma^2)p(\bbeta, \bgamma, \sigma^2\given  \by_{\bPhi}, \bPhi) d\bbeta d\bgamma d\sigma^2\;,    
\end{equation}
where $\mathbb{E}[\cdot]$ is the expectation with respect to the posterior distribution in (\ref{eq: svc_basic_compressed_bhm}). This is easily achieved using composition sampling, as outlined in Algorithm~\ref{alg2}.
\begin{algorithm}
	\caption{Parametric Inference from the Proposed Model}
	\label{alg1}
	\SetAlgoLined
    \Begin
    {    	
    	Draw $\tilde{\bgamma}_1\sim N({\boldsymbol 0},\bDelta)$ and $\tilde{\bgamma}_2\sim N({\boldsymbol 0},\bI_M)$\\
    	Set $\tilde{\bgamma}_3=\tilde{\bX}_{\bPhi}\bB \tilde{\bgamma}_1/\sigma+\tilde{\bgamma}_2$\\
	    Solve $(\tilde{\bX}_{\bPhi}\bB \bDelta \bB^{\T}\tilde{\bX}_{\bPhi}^{\T}/\sigma^2+\bI_M)\tilde{\bgamma}_4=\left(\left(\by_{\bPhi}-\bX_{\bPhi}\bbeta\right)/\sigma-\tilde{\bgamma}_3\right)$\\
           Set $\tilde{\bgamma}_5=\tilde{\bgamma}_1+\bDelta \bB^{\T}\tilde{\bX}_{\bPhi}^{\T}\tilde{\bgamma}_4/\sigma$.\\
           The resulting $\tilde{\bgamma}_5$ is a draw from the full conditional posterior distribution of $\bgamma$. The computation is dominated by step~(iii), which comprises $O(M^3+M^2H\tilde{P})$.\\
           When basis functions involve parameters, they are updated using Metropolis-Hastings steps since no closed form full conditionals are generally available for them.\\
    }
\end{algorithm}

\begin{algorithm}
	\caption{Predictive Inference from the Proposed Model}
	\label{alg2}
	\SetAlgoLined
    \Begin
    {
    $L$ denotes the number of post-convergence posterior samples.
    \For{$ l=1:L $}
    {
    	Draw $\{\bbeta^{(l)},\bgamma^{(l)}, \sigma^{2(l)}\}$ from (\ref{eq: svc_basic_compressed_bhm})\\
    	Draw $w_j(\bu_0)^{(l)}$ from $\bgamma^{(l)}$ using (\ref{basis})\\
	    Draw $y(\bu_0)^{(l)} \sim N(\sum_{p=1}^{P} x_p(\bu_0)\beta_p^{(l)}+\sum_{j=1}^{\tilde{P}}\tilde{x}_j(\bu_0)w_j(\bu_0)^{(l)},\sigma^{2(l)})$\\
           }
      $y(\bu_0)^{(1)},...,y(\bu_0)^{(L)}$ are samples from (\ref{eq: posterior_predictive_compressed_bhm}).  
    }
\end{algorithm}

The next section offers theoretical results on asymptotic consistency of the posterior distribution for the compressed model (\ref{eq: svc_basic_compressed_bhm}) and the posterior predictive distribution (\ref{eq: posterior_predictive_compressed_bhm}) with respect to the probability law for the uncompressed oracle model in (\ref{eq: svc_basic_general}).

%Note that $f(y(s_0)\given X_{\Phi},\tilde{X}_{\Phi},y_{\Phi}) = \int N(y(s_0)\given \sum_{p=1}^{P}x_p(s_0)\beta_p+\sum_{j=1}^{\tilde{P}}\tilde{x}_j(s_0)w_j(s_0),\sigma^2)p(\psi, \beta,\gamma,\sigma^2,\psi\given X_{\Phi},\tilde{X}_{\Phi},y_{\Phi})d\beta d\gamma d\sigma^2$, where $w_j(s_0)$ is as defined in (\ref{basis}). If $\{\beta^{(1)},\gamma^{(1)},\sigma^{(1)2}),...,(\beta^{(L)},\gamma^{(L)},\sigma^{2(L)}\}$ are the MCMC samples of $(\beta,\gamma,\sigma^2)^{\T}$ by fitting  (\ref{eq: svc_basic_compressed_bhm}), then the $l$-th approximate MCMC samples of $y(s_0)$, denoted by $y(s_0)^{(l)}$, is drawn from $N(\sum_{p=1}^{P}x_p(s_0)\beta_p+\sum_{j=1}^{\tilde{H}}\tilde{x}_j(s_0)w_j(s_0)^{(l)},\sigma^{2(l)})$.

\section{Posterior contraction from data sketching}\label{sec:theory}
\subsection{Definitions and Notations}
This section proves the posterior contraction properties of varying coefficients under the proposed framework. In what follows, we add a subscript $N$ to the compressed response vector $\by_{\bPhi,N}$, compressed predictor matrix $\tilde{\bX}_{\bPhi,N}$, dimension of the compression matrix $M_{N}$ and the number of basis functions $H_N$ to indicate that all of them increase with the sample size $N$. 
%This asymptotic paradigm is also meant to capture the fact that $M_N$ and $H_N$ are also dependent on the sample size $N$.
Naturally, the dimension of the basis coefficient vector $\bgamma$ and the compression matrix $\bPhi$ are also functions of $N$, though we keep this dependence implicit. Since we do not assume a functional variable selection framework, we keep $P$ fixed throughout, and not a function of $N$. We assume that $\bu_1,...,\bu_N$ follow i.i.d.
distribution $G$ on $\mathcal{D}$ with $G$ having a Lebesgue density $g$, which is bounded away from zero and infinity uniformly over $\mathcal{D}$.
The true regression function is also given by (\ref{eq: svc_basic_general}), with the true varying coefficients $w_1^*(\bu),...,w_{\tilde{P}}^*(\bu)$ belonging to the class of functions
\begin{align}
\mathcal{F}_{\xi}(\mathcal{D})=\{f:f\in L_2(\mathcal{D})\cap\mathcal{C}^{\xi}(\mathcal{D}),E_{\mathcal{U}}[|f|]<\infty\},
\end{align}
where $L_2(\mathcal{D})$ is the set of all square integrable functions on $\mathcal{D}$, $\mathcal{C}^{\xi}(\mathcal{D})$ is the class of at least $\xi$-times continuously differentiable functions in $\mathcal{D}$ and $E_{\mathcal{U}}$ denotes the expectation under the density of $g$. The probability and expectation under the true data generating model are denoted by $P^*$ and $E^*$, respectively. For algebraic simplicity, we make a few simplifying assumptions in the model. To be more specific, we assume that $\bbeta={\boldsymbol 0}$ and $\sigma^2=\sigma^{*2}$ is known and fixed at $1$. The first assumption is mild since $P$ does not vary with $N$ and we do not consider variable selection. The second assumption is also customary in asymptotic studies \citep{vaart2011information}. Furthermore, the theoretical results obtained by assuming $\sigma^2$ as a fixed value is equivalent to those obtained by assigning a prior with a bounded support on $\sigma^2$ \citep{van2009adaptive}. 

For a vector $\bv=(v_1,...,v_N)^{\T}$, we let $||\cdot||_1, ||\cdot||_{2}$ and $||\cdot||_{\infty}$ denote the $L_1, L_2$ and $L_{\infty}$ norms defined as  $||\bv||_2=(\sum_{n=1}^Nv_n^2)^{1/2}$, $||\bv||_1=\sum_{n=1}^N|v_n|$ and $||\bv||_{\infty}=\max_{n=1,..,N}|v_n|$, respectively. The number of nonzero elements in a vector is given by $||\cdot||_0$. In the case of a square integrable function $f(\bu)$ on $\mathcal{D}$, we denote the integrated $L_2-$norm of $f$ by $||f||_{2}=\left(\int_{\mathcal{D}} f(\bu)^2g(\bu)d\bu\right)^{1/2}$ and the sup-norm of $f$ by $||f||_{\infty}=\sup_{\bu\in\mathcal{D}}|f(\bu)|$. Thus $||\cdot||_{\infty}$ and $||\cdot||_2$ are used both for vectors and functions, and they should be interpreted based on the context. Finally, $e_{\min}(\bA)$ and $e_{\max}(\bA)$ represent the minimum and maximum eigenvalues of the symmetric matrix $\bA$, respectively. The Frobenius norm of the matrix $\bA$ is given by $||\bA||_F=\sqrt{\mbox{tr}(\bA^{\T}\bA)}$. For two nonnegative sequences $\{a_N\}$ and $\{b_N\}$, we write $a_N\asymp b_N$ to denote $0<\liminf_{N\rightarrow\infty}a_N/b_N\leq \limsup_{N\rightarrow\infty}a_N/b_N<\infty$. If $\lim_{N\rightarrow\infty}a_N/b_N=0$, we write $a_N=o(b_N)$ or $a_N\prec b_N$. We use $a_N\lesssim b_N$ or $a_N=O(b_N)$ to denote that for sufficiently large $N$, there exists a constant $C>0$ independent of $N$ such that $a_N\leq Cb_N$.

\subsection{Assumption, Framework and Main Results}
For simplicity, we assume that the random covariates $x_p(\bu)$, $p=1,...,P$ follow distributions which are independent of the distribution of the idiosyncratic error $\epsilon$. %We use $\mathbb{E}_{\mathcal{X}}$ to denote the expectation w.r.t. the covariate distribution.
We now state the following assumptions on the basis functions, $H_N, M_N$,  covariates and the sketching or compression matrix.
\begin{enumerate}[(A)]
\item For any $w_{j}^{\ast}(\bu)\in\mathcal{F}_{\xi}(\mathcal{D})$, there exists $\bgamma_{j}^*$ such that $$||w_j^{\ast}-\bB_j^{\T}\bgamma_j^{\ast}||_{\infty} = \sup\limits_{\bu\in\mathcal{D}}|w_j^*(\bu)-\sum_{h=1}^{H_N}B_{jh}(\bu)\gamma_{jh}^*|=O(H_N^{-\xi}), $$ for $j=1,...,\tilde{P}$, and $||\bgamma^{\ast}||_2^2\prec M_N^{d/(d+2\xi)}$.
\item $N,M_N,H_N$ satisfy $M_N=o(N)$ and $H_N\asymp M_N^{1/(2\xi+d)}$.
\item $||\bPhi\bPhi^{\T}-\bI_{M_N}||_F\leq C'\sqrt{M_N/N}$, for some constant $C'>0$, for all large $N$.
\item The random covariate $x_p(\bu)$ are uniformly bounded for all $\bu\in\mathcal{D}$, and w.l.g., $|x_p(\bu)|\leq 1$, for all $p=1,...,P$ and for all $\bu\in\mathcal{D}$.
\item There exists a sequence $\kappa_N$ such that 
$||\tilde{\bX}_{\bPhi,N}\balpha||^2\asymp\kappa_N||\tilde{\bX}_{N}\balpha||^2$, such that $1\prec N\kappa_N\prec M_N$ for any vector $\balpha\in\mathbb{R}^{N\tilde{P}}$.
\item For simplicity, assume $\bDelta=\bI$, $\bbeta=\boldsymbol{0}$, $\sigma^2$ is known and without loss of generality, $\sigma^2=1$.
\end{enumerate}

Assumption~(A) holds for orthogonal Legendre polynomials, Fourier series, B-splines and wavelets \citep{shen2015adaptive}. Assumption~(B) provides an upper bound on the growth of $M_N$ and $H_N$ as a function of $N$. Assumption~(C) is a mild assumption based on the theory of random matrices and occurs with probability at least $1-e^{-C''M_N}$ when $\bPhi$ is constructed using the Gaussian sketching for a constant $C''>0$ (see Lemma 5.36 and Remark 5.40 of \cite{vershynin2010introduction}). Assumption~(D) is a technical condition customarily used in functional regression analysis \citep{bai2019fast}. Assumption~(E) characterizes the class of feasible compression matrices, roughly explaining how the linear structure of the columns of the original predictor matrix is related to that of the compressed predictor matrix. Such an assumption is reasonable for the set of random compression matrices for a sequence $\kappa_N$ depending on $N$, $M_N$ and $\tilde{P}$ \citep{ahfock2017statistical}. As argued here, both Assumptions (C) and (E) can be proved to hold with high probability. We include them as assumptions because they are considered to hold with probability 1. This practice is common when random matrices are used in the study of computationally efficient Bayesian models \citep{guhaniyogi2015bayesian, guhaniyogi2016compressed, guhaniyogi2021sketching}, as it allows the focus to remain on model uncertainty without factoring in the uncertainty of random matrix construction. Finally, Assumption (F) is assumed for simplicity in mathematical derivation, and it could potentially be relaxed.

Let $\bw(\bu)=(w_1(\bu),...,w_{\tilde{P}}(\bu))^{\T}$ and $\bw^*(\bu)=(w_1^*(\bu),...,w_{\tilde{P}}^*(\bu))^{\T}$ be the $\tilde{P}$-dimensional fitted and true varying coefficients. Let
$\|\bw-\bw^*\|_2=\sum_{j=1}^{\tilde{P}}\|w_j-w_j^{\ast}\|_2$ denote the sum of integrated $L_2$ distances between the true and the fitted varying coefficients. Define the set 
$\mathcal{C}_N=\left\{\bw:||\bw-\bw^{*}||_2>\tilde{C}\theta_N\right\}$, for some constant $\tilde{C}$ and some sequence $\theta_N\rightarrow 0$ and $M_N\theta_N^2\rightarrow\infty$.
In addition, suppose $\pi_{N}(\cdot)$ and $\Pi_N(\cdot)$ are the prior and posterior densities of $w$ with $N$ observations, respectively. From equation (\ref{basis}), the prior distribution on $\bw$ is governed by the prior distribution on $\bgamma$, so that the posterior probability of $\mathcal{C}_N$ is
\begin{align*}
\Pi_N(\mathcal{C}_N|\by_{\bPhi,N},\tilde{\bX}_{\bPhi,N})=\frac{\int_{\mathcal{C}_N}f(\by_{\bPhi,N}|\tilde{\bX}_{\bPhi,N},\bgamma)\pi_{N}(\bgamma)}
                          {\int f(\by_{\bPhi, N}|\tilde{\bX}_{\bPhi,N},\bgamma)\pi_{N}(\bgamma)},
\end{align*}
where $f(\by_{\bPhi,N}|\tilde{\bX}_{\bPhi,N},\bgamma)$ is the joint density of $\by_{\bPhi,N}$ under model (\ref{eq: svc_basic_compressed_bhm}).
%This article intends to show
%\begin{align}\label{consistency1}
%\Pi_n(\mathcal{A}_n)\rightarrow 0,\:\:\mbox{a.s., when}\:\:n\rightarrow\infty.
%\end{align}
We begin with the following important result from the random matrix theory.

\begin{lemma}\label{lem1}
Consider the $M_N\times N$ compression matrix $\bPhi$ with each entry drawn independently from $N(0,1/N)$. Also, assume that $M_N=o(N)$. Then, almost surely
\begin{align}\label{compression_lemma}
(\sqrt{N}-o(\sqrt{N}))^2/N\leq e_{\min}(\bPhi\bPhi^{\T})\leq  e_{\max}(\bPhi\bPhi^{\T})\leq (\sqrt{N}+o(\sqrt{N}))^2/N,
\end{align}
when $N\rightarrow\infty$.
\end{lemma}
\begin{proof}
This is a consequence of Theorem~5.31 and Corollary~5.35 of \cite{vershynin2010introduction}.
\end{proof}
The inequalities in (\ref{compression_lemma}) are used to derive the following two results, which we present as Lemma~\ref{lem2}~and~\ref{lem3}.
\begin{lemma}\label{lem2}
Let $P^{\ast}$ denote the true probability distribution of $\by_N$ and $f^*(\by_{\bPhi,N}|\bgamma^*)$ denotes the density of $\by_{\bPhi,N}$ (omitting explicit dependence on $\tilde{\bX}_{\bPhi,N}$) under the true data generating model. Define
\begin{align}\label{eq:lem1}
\mathcal{A}_N=\left\{\by:\int \left\{f(\by_{\bPhi,N}|\bgamma)/f^*( \by_{\bPhi,N}|\bgamma^*)\right\} \pi_N(\bgamma)d\bgamma\leq \exp(-CM_N\theta_N^2)\right\}.
\end{align}
Then $P^*(\mathcal{A}_N)\rightarrow 0$ as $M_N,N\rightarrow\infty$ for any constant $C>0$.
\end{lemma}
\begin{proof}
See Section~\ref{sec:lem2} in the Appendix.
\end{proof}

\begin{lemma}\label{lem3}
%Consider testing the hypothesis $H_0:\gamma=\gamma^{\ast}$ vs. $H_1:\gamma\in\mathcal{B}_N^c$, where 
Let $\bgamma^{\ast}$ be any fixed vector in the support of $\bgamma$ and let 
$\mathcal{B}_N=\{\bgamma:||\bgamma-\bgamma^{\ast}||_2\leq C_{2w}\theta_NH_N^{1/2}\}$ for some constant $C_{2w}>0$. 
Then there exists a sequence %of test functions 
$\zeta_N$ % 
of random variables
depending on $\{\by_{\bPhi,N}, \bX_{\bPhi,N}\}$ and taking values in $(0,1)$ %
such that
\begin{align}
\mathbb{E}^{\ast}(\zeta_N)\lesssim \exp(-M_N\theta_N^2) \mbox{ and } \sup\limits_{\bgamma\in\mathcal{\bB}_N^c}\mathbb{E}_{\bgamma}(1-\zeta_N)\lesssim \exp(-M_N\theta_N^2),
\end{align}
where $\mathbb{E}_{\bgamma}$ and $\mathbb{E}^{\ast}$ denote the expectations under the distributions $f(\cdot\given\bgamma)$ and $f^{\ast}(\cdot\given\bgamma^{\ast})$, respectively.
\end{lemma}
\begin{proof}
See Section~\ref{sec:lem3} in the Appendix.
\end{proof}

We use the above results to establish the posterior contraction result for the proposed model.
\begin{theorem}\label{thm_main}
Under Assumptions~(A)-(F), our proposed model (\ref{eq: svc_basic_compressed_bhm}) satisfies $$\max_{j=1,...,\tilde{P}}\sup_{w_j^*\in\mathcal{F}_{\xi}(\mathcal{D})}\mathbb{E}^{\ast}\Pi_N(\mathcal{C}_N\given \by_{\bPhi,N},\tilde{\bX}_{\bPhi,N})\rightarrow 0, \text{as } N,M_N\rightarrow\infty,$$  with the posterior contraction rate $\theta_N\asymp M_N^{-\xi/(2\xi+d)}$.
\end{theorem}
\begin{proof}
See Section~\ref{sec:thm_main} in the Appendix for the detailed proof. Here we offer an outline of the proof. The steps are given below.\\
\emph{Step 1:} Using basis expansion of each $w_j$ and by Assumption (A), $\{\bw:||\bw-\bw^*||_2\geq\tilde{C}\theta_N\}\subset \{\bgamma:||\bgamma-\bgamma^*||_2H_N^{-1/2}\geq C_{2w}\theta_N\}=\mathcal{B}_N^c$, for some constant $C_{2w}>0$.\\
\emph{Step 2:} Consider the set $\mathcal{A}_N$ defined in Lemma~\ref{lem2} and the sequence of random variables $\zeta_N$ defined in Lemma~\ref{lem3}.
Note that $\mathbb{E}^*(\zeta_N)\rightarrow 0$ as $N\rightarrow\infty$, by Lemma~\ref{lem3} and $P^*(\mathcal{A}_N)\rightarrow 0$ as $N\rightarrow\infty$, by Lemma~\ref{lem2}.\\
\emph{Step 3:} We then consider the expression $\mathbb{E}^*[\Pi(\mathcal{B}_N^c\given \by_{\bPhi,N},\tilde{\bX}_{\bPhi,N})(1-\zeta_N)1_{y_N\in\mathcal{A}_N^c}]\\=\mathbb{E}^*\left[1_{\by_N\in\mathcal{A}_N^c}\frac{\left\{(1-\zeta_N)\int_{\mathcal{B}_N^c}\{f(\by_{\bPhi,N}|\bgamma)/f^*( \by_{\bPhi,N}|\bgamma^*)\}\pi_N(\bgamma)d\bgamma\right\}}{\left\{\int\{f(\by_{\bPhi,N}|\bgamma)/f^*( \by_{\bPhi,N}|\bgamma^*)\}\pi_N(\bgamma)d\bgamma\right\}}\right]$. Due to Step 2, it only suffices to show that this expression converges to $0$ as $N\rightarrow\infty$.\\
\emph{Step 4:} Under Assumptions (A)-(F), the numerator of the above expression decays exponentially to $0$ as a function of $M_N\theta_N^2$. The inverse of the denominator grows at a slower rate of $M_N\theta_N^2$. The result is then proved by considering $M_N\theta_N^2=o(N).$ 
\end{proof}

Since $\theta_N\rightarrow 0$ as $N\rightarrow\infty$, the model consistently estimates the true varying coefficients under the integrated $L_2$-norm. Further, data compression decreases the effective sample size from $N$ to $M_N$, hence, the contraction rate $\theta_N$ obtained in Theorem~\ref{thm_main} is optimal and adaptive to the smoothness of the true varying coefficients. Our next theorem justifies the two-stage prediction strategy described in Section~\ref{sec:eff_post}. 
\begin{theorem}\label{thm_main2}
For any input $\bu_0$ drawn randomly with the density $g$ and corresponding predictors $\tilde{x}_1(\bu_0),\ldots,\tilde{x}_{\tilde{P}}(\bu_0)$, let $f_u$ be the predictive density $p(y(\bu_0)\given \tilde{x}_1(\bu_0),\ldots,\tilde{x}_{\tilde{P}}(\bu_0),w(\bu_0))$ derived from (\ref{eq: svc_basic_general}) without data compression. Let $f^{\ast}$ be the true data generating model (i.e., (\ref{eq: svc_basic_general}) with $\bw(\bu_0)$ fixed at $\bw^{\ast}(\bu_0)$). Given $\bu_0$ and $\tilde{x}_1(\bu_0),\ldots,\tilde{x}_{\tilde{P}}(\bu_0)$, define $h(f_u,f^{\ast})=\int (\sqrt{f_u}-\sqrt{f^{\ast}})^2$ as the Hellinger distance between the densities $f_u$ and $f^{\ast}$. Then 
\begin{align}
\mathbb{E}^{\ast}\mathbb{E} \mathbb{E}_{\mathcal{U}}[h(f_u,f^*)\given \tilde{\bX}_{\bPhi,N},\by_{\bPhi,N}]\rightarrow 0,\:\:\mbox{as}\:\:N,M_N\rightarrow\infty,
\end{align}
where $\mathbb{E}_{\mathcal{U}}$, $\mathbb{E}$ and $\mathbb{E}^{\ast}$ stand for expectations with respect to the density $g$, the posterior density $\Pi_N(\cdot|\tilde{\bX}_{\bPhi,N},\by_{\bPhi,N})$ and the true data generating distribution, respectively.
\end{theorem}
\begin{proof}
See Section~\ref{sec:thm_main2} in the Appendix.
\end{proof}
The theorem states that the predictive density of the VCM model in (\ref{eq: svc_basic_general}) is arbitrarily close to the true predictive density even when we plug-in inference on parameters from (\ref{eq: svc_basic_compressed_bhm}). 

\section{Simulation Results}\label{sec:empirical}
\subsection{Inferential performance}\label{sec: inferential_performance_sim}

We empirically validate our proposed approach using (\ref{eq: svc_basic_compressed_bhm}) for $d=2$, i.e., for the spatially varying coefficient models. The approach, henceforth abbreviated as \textit{geoS}, is compared with the uncompressed model (\ref{eq: stack_model}) on some simulated data in terms of inferential performance and computational efficiency.  We simulate data by using a fixed set of spatial locations $\bu_1,\ldots,\bu_N$ that were drawn uniformly over the domain $\mathcal{D} = [0,1]\times [0,1]$. We set $\tilde{P}=P=3$ and assume $\bbeta=0$, i.e., all predictors have purely space-varying coefficients. We set $\tilde{x}_1(\bu_i)=1$, for all $i=1,\ldots,N$, while the values of $\tilde{x}_j(\bu_1),\ldots,\tilde{x}_j(\bu_N)$ for $j=2,3$ were set to independently values from $N(0,1)$. For each $n=1,\ldots,N$, the response $y(\bu_n)$ is drawn independently from $N(w_1^*(\bu_n)+w_2^*(\bu_n)\tilde{x}_2(\bu_n)+w_3^*(\bu_n)\tilde{x}_3(\bu_n),\sigma^{*2})$ following (\ref{eq: stack_model}), where $\sigma^{*2}$ is set to be $0.1$. The true space-varying coefficients ($w^*_j(\bu)$s) are simulated from a Gaussian process with mean $0$ and covariance kernel $C(\cdot,\cdot;\theta_j)$, i.e., $(w_j^*(\bu_1),...,w_j^*(\bu_N))^{\T}$ is drawn from $N(0,C^*(\theta_j))$, for each $j=1,\ldots,\tilde{P}$, where $C^*(\theta_j)$ is an $N\times N$ matrix with the $(n,n')$th element $C(\bu_n,\bu_{n'};\theta_j)$. We set the covariance kernel $C(\cdot,\cdot;\theta_j)$ to be the exponential covariance function given by

\begin{align}\label{eq: cov_fn}
C(\bu,\bu';\theta_j) = \delta^2_j \exp\left\lbrace -\frac{1}{2} \left( \frac{||\bu-\bu'||}{\phi_j}\right) \right\rbrace,\:\:j=1,2,3,
\end{align}
with the true values of $\delta_1^2,\delta_2^2,\delta_3^2$ set to $1, 0.8, 1.1$, respectively. We fix the true values of $\phi_1,\phi_2,\phi_3$ at $1, 1.25, 2$, respectively.

While fitting \emph{geoS} and its uncompressed analogue (\ref{eq: stack_model}), the varying coefficients are modeled through the linear combination of $H$ basis functions as in (\ref{basis}), where these basis functions are chosen as the tensor-product of B-spline bases of order $q=4$ \citep{shen2015adaptive}. More specifically, for $\bu=(u^{(1)},u^{(2)})$, the $j$-th varying coefficient is modeled as 
\begin{align}\label{eq: tensor_prod}
w_j(\bu)=\sum_{h_1=1}^{H_1}\sum_{h_2=1}^{H_2}B_{jh_1}^{(1)}(u^{(1)})B_{jh_2}^{(2)}(u^{(2)})\gamma_{jh_1h_2}\;,
\end{align}
where the marginal B-splines $B_{jh_1}^{(1)}$, $B_{jh_2}^{(2)}$ are defined on sets of $H_1$ and $H_2$ knots, respectively. The knots are chosen to be equally-spaced so the entire set of $H= H_1H_2$ knots is uniformly spaced over the domain $\mathcal{D}$. We complete the hierarchical specification by assigning independent $IG(2, 0.1)$ priors (mean $0.1$ with infinite variance) for $\sigma^2$ and $\tau_j^2$ for each $j=1,\ldots,P$.  

We implemented our models in the \texttt{R} statistical computing environment on a Dell XPS 13 PC with Intel Core i7-8550U CPU @  4.00GHz  processors at 16 GB of RAM. For each of our simulation data sets we ran a single-threaded MCMC chain for 5000 iterations. Posterior inference was based upon 2000 samples retained after adequate convergence was diagnosed using Monte Carlo standard errors and effective sample sizes (ESS) using the \texttt{mcmcse} package in \texttt{R}. Source codes for these experiments are available from \url{https://github.com/LauraBaracaldo/Bayesian-Data-Sketching-in-Spatial-Regression-Models}.

{
\small
\begin{table}[ht]
\centering
\scalebox{0.8}{
\begin{tabular}{p{4cm}p{4cm}p{2.5cm}}
  \hline
  \hline
   & \multicolumn{2}{c}{$N=5000$, $H = 225$, $K=50$ } \\
   \cline{2-3}
   \\[-10pt]
   & \multicolumn{1}{l}{\textit{(geoS)} $M=700$} & \textit{Uncompressed} \\
   \hline
   $MSE$ \textit{(SVC)} & 0.0335 (0.028, 0.039) & 0.0109 \\
   \textit{95\% CI length} & 0.6751 (0.654, 0.723) & 0.2441 \\
   \textit{95\% CI coverage} & 0.9406 (0.913, 0.959) & 0.9520 \\
   $MSPE$ & 0.1986 (0.174, 0.231) & 0.1369 \\
   \textit{95\% PI length} & 1.6449 (1.567, 1.755) & 1.3071 \\
   \textit{95\% PI coverage} & 0.9400 (0.912, 0.962) & 0.9380 \\
   \textit{Computation efficiency} & 2.1165 (1.643, 2.152) & 0.6298\\
   \hline
\end{tabular}}

\vspace{1em} % Add some vertical space between the tables

\scalebox{0.8}{
\begin{tabular}{p{4cm}p{4cm}p{2.5cm}}
  \hline
  \hline
   & \multicolumn{2}{c}{$N=10000$, $H = 256$, $K=50$} \\
   \cline{2-3}
   \\[-10pt]
   & \multicolumn{1}{l}{\textit{(geoS)} $M=1000$} & \textit{Uncompressed} \\
   \hline
   $MSE$ \textit{(SVC)} & 0.0238 (0.019, 0.028)& 0.0092 \\
   \textit{95\% CI length} & 0.6101 (0.591, 0.631) & 0.2837 \\
   \textit{95\% CI coverage} & 0.9253 (0.920, 0.960) & 0.9500 \\
   $MSPE$ & 0.1737 (0.156, 0.191) & 0.1260 \\
   \textit{95\% PI length} & 1.6013 (1.534, 1.653) & 1.3770 \\
   \textit{95\% PI coverage} & 0.9460 (0.928, 0.965) & 0.9510 \\
   \textit{Computation efficiency} & 2.2368 (2.101, 2.288) & 0.4981 \\
   \hline
   \hline
\end{tabular}}

\vspace{1em} % Add some vertical space between the tables

\scalebox{0.8}{
\begin{tabular}{p{4cm}p{4cm}p{2.5cm}}
  \hline
  \hline
   & \multicolumn{2}{c}{$N=100000$, $H = 400$, $K=10$} \\
   \cline{2-3}
   \\[-10pt]
   & \multicolumn{1}{l}{\textit{(geoS)} $M=3200$} & \textit{Uncompressed} \\
   \hline
   $MSE$ \textit{(SVC)} & 0.0067 (0.003, 0.008) & 0.0008 \\
   \textit{95\% CI length} & 0.3007 (0.221, 0.310) & 0.1712 \\
   \textit{95\% CI coverage} & 0.9360 (0.926, 0.941) & 0.953 \\
   $MSPE$ & 0.1242 (0.115, 0.131)  & 0.112 \\
   \textit{95\% PI length} & 1.3503 (1.290, 1.381) & 1.239 \\
   \textit{95\% PI coverage} & 0.9510 (0.942, 0.956) & 0.937 \\
   \textit{Computation efficiency} & 5.9081 (5.814, 6.001) & 0.981 \\
   \hline
   \hline
\end{tabular}}
\caption{Summary results: 50\% (2.5\%, 97.5\%) over $K$ simulations for the compressed \textit{geoS} model. Median values for each metric over $K$ simulations are presented for the uncompressed model. Mean Squared Error (MSE), length and coverage of 95\% CI for the spatially varying coefficients are presented. We also provide mean squared prediction error (MSPE), coverage and length of 95\% predictive intervals for competing models.}
\label{table:tablesim1}
\end{table}

}

Table~\ref{table:tablesim1} summarizes the estimates of the varying coefficients and the predictive performance for \textit{geoS} in comparison to the uncompressed model. These results are based on $K$ independently generated data sets for each scenario, with $N=5,000$ (case 1), $N=10,000$ (case 2), and $N=100,000$ (case 3). For each case, the compressed dimension is taken to be $M\approx 10\sqrt{N}$, which seems to be effective from empirical considerations in our simulations. We provide further empirical justification for this choice in Section~\ref{sec: dimension_compression}. Our \textit{geoS} approach compresses the sample sizes to $M=700$, $M=1000$  and $M=3000$ in cases~1, 2 and 3, respectively. The number of fitted basis functions in cases 1, 2~and~3 are $H=225, 256, 400$, respectively. For each simulated data, we evaluated 6 model assessment metrics and 1 computational efficiency metric that are listed in Table~\ref{table:tablesim1}. We present the median, 2.5, and 97.5 quantiles for each of the metrics on the $K$ data sets. We see that \textit{geoS} offers competitive inferential performance and outperforms the uncompressed model. For example, the confidence interval for some of the metrics, while including the values for the uncompressed model, often reveal heavier mass to the left of the value for the uncompressed model.  

\begin{figure}[t]
    \centering
    \includegraphics[width=11cm]{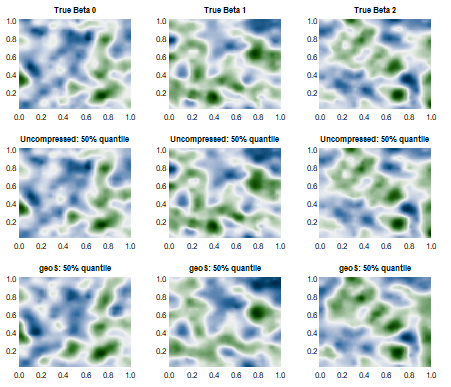}
    \caption{Simulation case 1: $(N, H)=(5000, 225)$. Two-dimensional true and predicted surfaces over the unit square $\mathcal{D} = [0,1]\times [0,1]$. First row corresponds to the surfaces of true space-varying coefficients $\beta^*_p(s)$, $p=1, 2, 3$. Rows 2 and 3 correspond to the predicted 50\% quantile surfaces for the uncompressed and compressed \textit{geoS} models respectively.}
    \label{fig:SufBetas1}
\end{figure}

\begin{figure}[t]
    \centering
    \includegraphics[width=11cm]{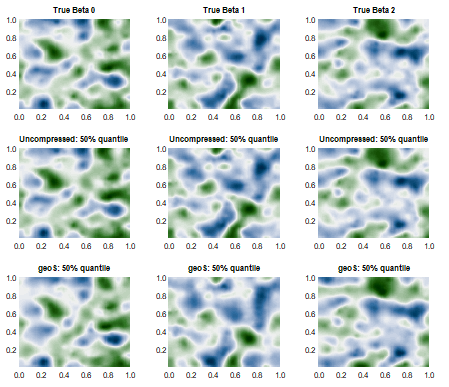}
    \caption{Simulation case 1: $(N, H)=(10000, 256)$. Two-dimensional true and predicted surfaces over the unit square $\mathcal{D} = [0,1]\times [0,1]$. First row corresponds to the surfaces of true space-varying coefficients $\beta^*_p(s)$, $p=1, 2, 3$. Rows 2 and 3 correspond to the predicted 50\% quantile surfaces for the uncompressed and compressed \textit{geoS} models respectively.}
    \label{fig:SufBetas2}
\end{figure}

Figures~\ref{fig:SufBetas1}~and~\ref{fig:SufBetas2} present the estimated varying coefficients {in one representative simulation experiment} by \emph{geoS} and the uncompressed data model for cases~1~and~2, respectively. These figures reveal point estimates that are substantively similar to those from \emph{geoS} and the uncompressed model. The mean squared error of estimating varying coefficients, defined as  $\sum_{j=1}^3\sum_{n=1}^N(\widehat{w}_j(\bu_n)-w_j^*(\bu_n))^2/(3N)$ (where $\widehat{w}_j(\bu_n)$ is the posterior median of $w_j(\bu_n)$), also confirms very similar point estimates offered by the compressed and uncompressed models (see Table~\ref{table:tablesim1}). Further, \textit{geoS} offers close to nominal coverage for 95\% credible intervals for varying coefficients, with little wider credible intervals compared to uncompressed data model. This can be explained by the smaller sample size for the \textit{geoS} model, though the difference turns out to be minimal. We also carry out predictive inference using \textit{geoS} (Section~\ref{sec:eff_post}). Table~\ref{table:tablesim1} presents mean squared predictive error (MSPE), average length and coverage for the 95\% predictive intervals, based on $N^*=500$ out of the sample observations. We find \emph{geoS} delivers posterior predictive estimates and predictive coverage that are very consistent with the uncompressed model, perhaps with marginally wider predictive intervals than those without compression. Finally, the computational efficiency of both models are computed based on the metric $\log_2(ESS/\mbox{Computation Time})$, where $ESS$ denotes the effective sample size averaged over the MCMC samples of all parameters. We find \textit{geoS} is almost $240\%$, $350\%$ and $500\%$ more efficient than the uncompressed model for $N=5,000$, $N=10,000$  and $N=100,000$, respectively, while delivering substantively consistent inference on the spatial effects.

 {
 We conduct an additional experiment where we compare our method with a sparse Gaussian Process (GP) model. For each $n=1,\ldots$ the simulated response $y(\bu_n)$ is drawn independently from $N(w^*(\bu_n),\sigma^{*2})$. Table~\ref{table:sparseGP} summarizes results in terms of inferential performance and computational efficiency for \textit{geoS} compared to the \textit{predictive process model} \citep[e.g.,][]{banerjee2008gaussian}. The predictive process achieves a reduction in computational complexity by projecting the original Gaussian Process (GP) onto a lower-dimensional subspace defined by a set of knots or inducing points, and is envisioned as a sparse GP model. It is implemented using the \texttt{spBayes} package in \texttt{R}. Notably, the \texttt{spBayes} package only allows fitting a varying intercept model with GP or predictive process fitted on the varying intercept, which prompted us to simulate the data as above. Our findings indicate that our method and the predictive process exhibit very similar inferential performance in terms of accuracy and predictive capability. However, $geoS$ demonstrates superior computational efficiency across all evaluated scenarios. Unlike the sparse GP, which experiences increased computational demands with larger $M$, \textit{geoS} maintains a consistent computational profile in terms of scalability and offers a more practical method for handling very large datasets.
 }

\begin{table}[!ht]
\centering
\begin{tabular}{c|c|p{2cm}|p{2cm}}
\hline
    & & $GeoS$ & $Sparse GP$ \\ \hline
\multirow{4}{*}{  $M=710$ } & $MSPE$ & 0.110  & 0.103  \\ 
                       & \textit{95\% PI length} & 1.320 & 1.503\\
                       & \textit{95\% PI coverage} & 0.940 & 0.942 \\
                       & \textit{Time (secs)}  & 17.47   & 5190.77\ \\ \hline
\multirow{4}{*}{$M=337$ } & $MSPE$ & 0.120  & 0.111 \\ 
                           & \textit{95\% PI length} & 1.396 & 1.316 \\
                       & \textit{95\% PI coverage} &  0.932 & 0.944 \\

                       & \textit{Time (secs)}&  14.23 & 1217.64 \\ \hline
\multirow{4}{*}{ $M=142$ } & $MSPE$ & 0.146 & 0.128 \\  
& \textit{95\% PI length} & 1.561 & 1.414\\
                       & \textit{95\% PI coverage} & 0.918 & 0.931\\
                       & \textit{Time (secs)}& 12.121  & 242.67 \\ \hline
\end{tabular}
\caption{Performance comparison of GeoS and Predictive Process for different values of $M =k\sqrt{N}$, for $k=2,5,10$ and $N=5000$}
\label{table:sparseGP}
\end{table}

\subsection{Choice of the dimension of the compression matrix $M$}\label{sec: dimension_compression}
We present investigations into the appropriate compression matrix size $M$. For simulated data with sample size $N=100000$, we ran our model for different values of $M=k\sqrt{N}$, $k=1,\ldots, 20$. Figure~\ref{fig:Mchoice} shows the variations in the prediction of points and intervals reflected in the $MSPE$ and 95\% predicted interval coverage and length, respectively. Unsurprisingly, as $M$ increases the MSPE drops with a decreased rate of decline until $k\sim 10$. In terms of interval prediction, the predictive coverage seems to oscillate within the narrow interval $(0.9, 0.97)$ for all values of $M$, but the length of the predictive interval improves as $M$ increases and begins to stabilize around $k\sim 10$. We observe that the choice of $M\sim 10\sqrt{N}$ leads to good performance in various simulations and real data analysis.

\begin{figure}[t]
  \centering
  
   \begin{subfigure}[b]{0.42\linewidth}
    \includegraphics[width=\linewidth]{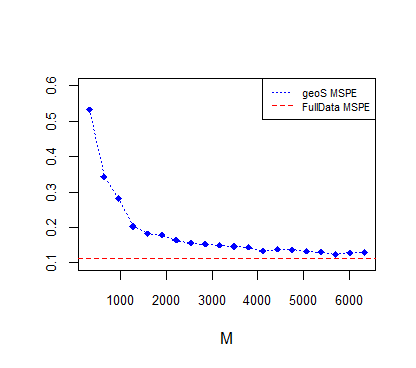}
     \caption{}
       \end{subfigure}\hspace{0.01\textwidth}%
       \centering
   \begin{subfigure}[b]{0.44\linewidth}
    \includegraphics[width=\linewidth]{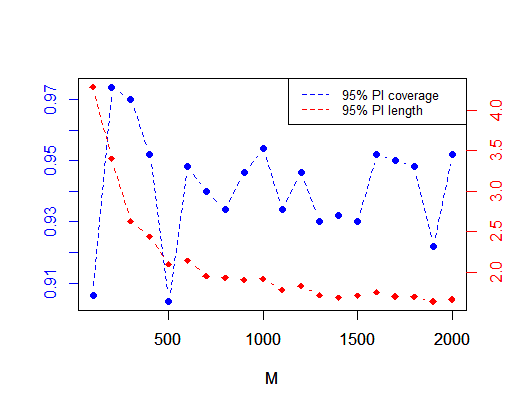}
     \caption{}
  \end{subfigure}
  
    \caption{(a) MSPE, (b) 95\% predictive interval coverage and length for different choices of $M$}
  \label{fig:Mchoice}
\end{figure}

\section{Vegetation Data Analysis}\label{sec:empirical2}

We implement \emph{geoS} to analyze vegetation data gathered through the Moderate Resolution Imaging Spectroradiometer (MODIS), which resides aboard the Terra and Aqua platforms on NASA spacecrafts. MODIS vegetation indices, produced on 16-day intervals and at multiple spatial resolutions, provide consistent information on the spatial distribution of vegetation canopy greenness, a composite property of leaf area, chlorophyll and canopy structure. The variable of interest will be the Normalized Difference Vegetation Index (NDVI), which quantifies the relative vegetation density for each pixel in a satellite image, by measuring the difference between the reflection in the near-infrared spectrum (NIR) and the red light reflection (RED): $NDVI=\frac{NIR - RED}{NIR +RED}$. High NDVI values, ranging between $0.6$ and $0.9$ indicate high density of green leaves and healthy vegetation, whereas low values, 0.1 or below, correspond to low or absence of vegetation as in the case of urbanized areas. When analyzed over different locations, NDVI can reveal changes in vegetation due to human activities such as deforestation and natural phenomena such as wild fires and floods. 

We analyze geographical data mapped on a projected sinusoidal grid (SIN), located on the western coast of the United States, more precisely zone~\textit{h08v05}, between $30^\circ N$ to $40^\circ N$ latitude and $104^\circ W$ to $130 ^\circ W$ longitude (see Figure~\ref{fig:DataNDVI2}(a)). The data, which were downloaded using the \texttt{R} package MODIS, comprises $133,000$ observed locations where the response was measured using the MODIS tool over a 16-day period in April 2016.%  2016.04.06 to 2016.04.21
. We retained $N = 113,000$ observations (randomly chosen) for model fitting and used the rest for prediction. In order to fit (\ref{eq: svc_basic_general}), we set $y(\bu_n)$ to be the transformed NDVI ($\log(NDVI)+1$), $P=\tilde{P}=2$ and consider the $P \times 1$ vector of predictors that includes an intercept and a binary index of urban area, both with fixed effects and spatially varying coefficients, i.e., $\bx(\bu_n) = \tilde{\bx}(\bu_n) = (1, \;  x_2(\bu_n))^{\T}$, with $x_2(\bu_n) = \mathbbm{1}_U(\bu_n)$, where $U$ denotes an urban area. 
%This yields to $y(s_n) \sim N(\beta_1 + \beta_2x_2(s_n) + w_1^*(s_n)+w_2^*(s_n)x_2(s_n),\sigma^{*2})$.

As in Section~\ref{sec:empirical}, we fit \emph{geoS} with $M\sim 10\sqrt{N}=2300$ and its uncompressed counterpart (\ref{eq: stack_model}), by modeling the varying coefficients through a linear combination of basis functions constructed using the tensor product of B-splines of order $q=4$ as in (\ref{eq: tensor_prod}). We set $H = H_1H_2= 39^2=1521$ uniformly distributed knots in the domain $\mathcal{D}$, which results in $HP=3042$ basis coefficients $\gamma_{jh}$ that are estimated. The specification of the priors are identical to the simulation studies for $\sigma^2$, and $\tau^2_j$, while $\beta_j$ is assigned a flat prior for $j=1,\ldots,P$.
%As customary, we finish up the data analysis setting by assigning prior distributions for parameters $\beta_j$, $\tau^2_j$ and $\sigma^{*2}$; $j=1,2$. For the intercept and slope of the spatially static components, $\beta_1$ and $\beta_2$ respectively, we assign flat normal priors, whereas for parameters $\tau^2_j$; $j=1,2$ and $\sigma^{*2}$ we assign an inverse gamma prior $IG(2, 0.1)$.

{\small
\begin{table}[t]
\caption{Median and 95\% credible interval of $\beta_1,\beta_2$ for geoS and its uncompressed analogue are presented for the Vegetation data analysis. We also present MSPE, coverage and length of 95\% predictive intervals for the competing models. Computational efficiency for the two competing models are also provided.}
\label{table:tabledata1}
\centering
\scalebox{0.9}{\begin{tabular}{ccc}
  \hline
    \hline
   &  \multicolumn{1}{c}{ \textit{(geoS)}  $M= 2300$ }
  &                                          
\multicolumn{1}{c}{\textit{Uncompressed}} \\
\cline{2-3}
\\[-10pt]
\hline
$\beta_1$  & 0.222 (0.212, 0.230)  &  0.229 (0.219, 0.237) \\
$\beta_2$  & -0.060 (-0.074, -0.047)  &  -0.071 (-0.082, -0.059)  \\
 $MSPE$   & 0.00327 &  0.00276  \\
 \textit{95\% PI length} & 0.23445  &  0.22136   \\
 \textit{ 95\% PI coverage } & 0.95250  & 0.95411  \\
 % \textit{Time in min.} (1K iter)&     &     \\
  \textit{Computation efficiency} & 3.5424   &  0.46901   \\   \hline
   \hline
\end{tabular}}
\end{table}
}

We ran an MCMC chain for 5000 iterations and retained 2000 samples for posterior inference after adequate convergence was diagnosed.  The posterior mean of $\beta_1$ and $\beta_2$, along with their estimated 95\% credible intervals corresponding to \emph{geoS} and the uncompressed model are presented in Table~\ref{table:tabledata1}. Additionally, Table~\ref{table:tabledata1} offers predictive inference from both competitors based on $N^*=20,000$ test observations. According to both models there is a global pattern of relatively low vegetation density for areas with positive urban index as the estimated slope coefficient $\beta_2$ is negative in the compressed \emph{geoS} and in the uncompressed models. In terms of point prediction and quantification of predictive uncertainty, the two competitors offer practically indistinguishable results, as revealed by Table~\ref{table:tabledata1}. 

{\small
\begin{figure}[htbp]
  \centering
   \hspace{0.1\textwidth}%
   \begin{subfigure}[b]{0.25\linewidth}
    \includegraphics[width=\linewidth]{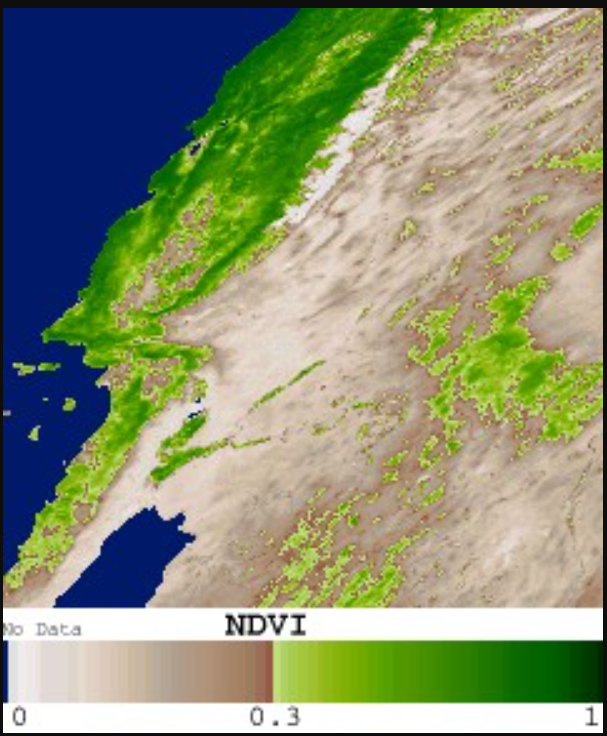}
     \caption{}
       \end{subfigure}\hspace{0.1\textwidth}%
       \centering
   \begin{subfigure}[b]{0.45\linewidth}
    \includegraphics[width=\linewidth]{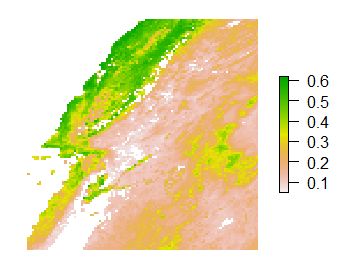}
     \caption{}
  \end{subfigure}
  \begin{subfigure}[b]{0.3\linewidth}
    \includegraphics[width=\linewidth]{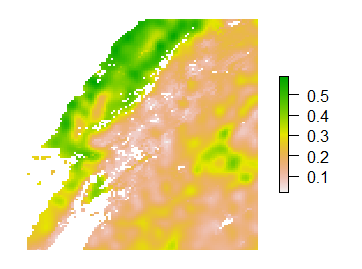}
    \caption{}
  \end{subfigure}
   \begin{subfigure}[b]{0.3\linewidth}
    \includegraphics[width=\linewidth]{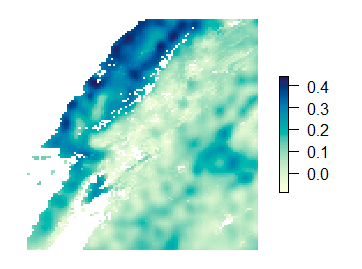}
    \caption{}
  \end{subfigure}
   \begin{subfigure}[b]{0.3\linewidth}
    \includegraphics[width=\linewidth]{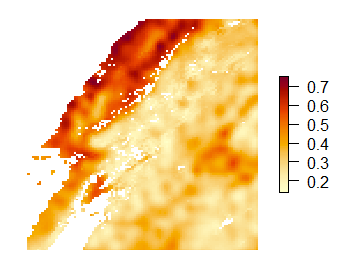}
    \caption{}
  \end{subfigure}
  \begin{subfigure}[b]{0.3\linewidth}
    \includegraphics[width=\linewidth]{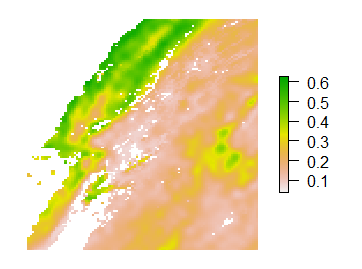}
    \caption{}
  \end{subfigure}
    \begin{subfigure}[b]{0.3\linewidth}
    \includegraphics[width=\linewidth]{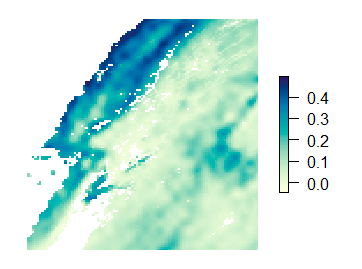}
     \caption{}
  \end{subfigure}
   \begin{subfigure}[b]{0.3\linewidth}
    \includegraphics[width=\linewidth]{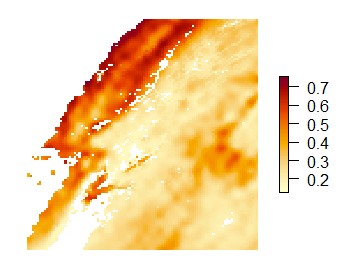}
    \caption{}
  \end{subfigure}
    \caption{Colored NDVI images of western United States (zone h08v05). (a) Satellite image: MODIS/Terra Vegetation Indices 16-Day L3 Global 1 km SIN Grid - 2016.04.06 to 2016.04.21; (b) True NDVI surface (raw data). Figures (c), (d) \& (e) present NVDI predicted 50\%, 2.5\% and 97.5\% quantiles for the \textit{geoS} model. Figures (f), (g) \& (h) present  NVDI Predicted 50\%, 2.5\% and 97.5\% quantiles for the uncompressed model.}
  \label{fig:DataNDVI2}
\end{figure}
}

Further, Figure \ref{fig:DataNDVI2} shows that the 2.5\%, 50\% and 97.5\% quantiles for the posterior predictive distribution are almost identical for the two competitors across the spatial domain, with the exception of neighborhoods around locations having lower NDVI values. Notably, \emph{geoS} offers nominal coverage for 95\% prediction intervals, even with a significant reduction in the sample size from $N=113,000$ to $M=2300$. Data sketching to such a scale considerably reduces the computation time, leading to a much higher computation efficiency of \emph{geoS} in comparison with its uncompressed analogue. %this was quantified and presented in table \ref{table:tabledata1}, with values of $3.5424$ and $0.46901$ for the \emph{geoS} and uncompressed models respectively.

\section{Summary}\label{sec: summary}
We have developed Bayesian sketching for functional response and predictor variables using varying coefficient regression models. The method achieves dimension reduction by compressing the data using a random linear transformation. The approach is different from the prevalent methods for large functional data in that no new models or algorithms need to be developed since those available for existing varying coefficient regression models can be directly applied to the compressed data. We establish attractive concentration properties of the posterior and posterior predictive distributions and empirically demonstrate the effectiveness of this method for analyzing large functional data sets. %Access to the values of the response and predictors in the full data are not required at stage of inference, which preserves data confidentiality should that be of concern in the application. 

\acks{Rajarshi Guhaniyogi acknowledges funding from the National Science Foundation through DMS-2220840 and DMS-2210672; funding from the Office of Naval Research through N00014-18-1-274; and funding from the National Institute of Neurological Disorders and Stroke (NIH/NINDS) through R01NS131604. Sudipto Banerjee acknowledges funding from the National Science Foundation through DMS-1916349 and DMS-2113778; and funding from the National Institute of Health through NIEHS-R01ES027027 and NIEHS-R01ES030210.}

\appendix

%We present theoretical results building up to the proofs of Theorems~\ref{thm_main}~and~\ref{thm_main2}. Lemma~\ref{lem1} states an important result from random matrix theory that is easily obtained from Theorem 5.31 and Corollary 5.35 of \cite{vershynin2010introduction}. We prove Lemmas~\ref{lem2}~and~\ref{lem3}. The results in Lemma \ref{lem1}-\ref{lem3} are further used to prove Theorems~\ref{thm_main}~and~\ref{thm_main2}.

\section{Proof of Lemma~\ref{lem2}}\label{sec:lem2}
\begin{comment}
\begin{lemma}\label{lem2}
Let $P^{\ast}$ denote the true probability distribution of $\by_N$ and $f^{\ast}(\by_{\bPhi,N}|\bgamma^*)$ denotes the density of $\by_{\bPhi,N}$ under the true data generating model. Define
\begin{align}\label{eq:lem1}
\mathcal{A}_N=\left\{\by:\int \left\{f(\by_{\bPhi,N}|\bgamma)/f^*( \by_{\bPhi,N}|\bgamma^*)\right\} \pi_N(\bgamma)d\bgamma\leq \exp(-C_1M_N\theta_N^2)\right\}.
\end{align}
Then $P^{\ast}(\mathcal{A}_N)\rightarrow 0$ as $N\rightarrow\infty$,\:\:\mbox{for any constant $C_1>0$}.
\end{lemma}
\end{comment}
\begin{proof}
Define
\begin{align}
\mathcal{A}_{1N}=\left\{K(f^*,f)\leq M_N\theta_N^2,\:\:V(f^*,f)\leq M_N\theta_N^2\right\}.
\end{align}
By Lemma~10 in \cite{ghosal2007convergence}, to show (\ref{eq:lem1}) it is enough to show that $\Pi(\mathcal{A}_{1N})\gtrsim \exp(-C_2M_N\theta_N^2)$ for some constant $C_2>0$.
Let $e_k$, $1\leq k\leq M_N$ be the ordered eigenvalues of $(\bPhi\bPhi^{\T})^{-1}$. After some calculations, we derive the following expressions,
\begin{align}\label{eq2}
K(f^{\ast},f) &= \frac{1}{2}\left\{\sum_{k=1}^{M_N}(e_k-1-\log(e_k)) + \mathbb{E}_{\mathcal{U}}\mathbb{E}_{\mathcal{\bX}}\left[|| \tilde{\bX}_{\bPhi,N}\bB(\bgamma-\bgamma^*)- \tilde{\bX}_{\bPhi,N}\bet^*||_2^2\right]\right\} \mbox{ and } \nonumber\\
V(f^*,f) &= \sum_{k=1}^{M_N}\frac{(1-e_k)^2}{2} + \mathbb{E}_{\mathcal{U}}E_{\mathcal{\bX}}\left[||(\bPhi\bPhi^{\T})^{-1}( \tilde{\bX}_{\bPhi,N}\bB(\bgamma-\bgamma^*)-\tilde{\bX}_{\bPhi,N}\bet^*)||_2^2\right],
\end{align}
where $\bet^* = (\bet^*(\bu_1)^{\T},...,\bet^*(\bu_N)^{\T})^{\T}$, $\bet^{\ast}(\bu) = (\eta_1^{\ast}(\bu),\ldots,\eta_{\tilde{P}}^{\ast}(\bu))^{\T}$, and $\eta_j^{\ast}(\bu) = w_j^*(\bu)-\sum_{h=1}^{H_N}B_{jh}(\bu)\gamma_{jh}^*$. Expanding $\log(e_k)$ in the powers of $(1-e_k)$ and using Lemma~1 in \cite{jeong2020unified} we find $(e_k-1-\log(e_k))\sim (1-e_k)^2/2$. Another use of Lemma~1 in \cite{jeong2020unified} yields $\sum_{k=1}^{M_N}(1-e_k)^2\lesssim ||\bI-\bPhi\bPhi^{\T}||_F^2\lesssim M_N/N\leq M_N\theta_N^2$. Using Lemma~\ref{lem1}, $e_k\asymp 1$ for all $k=1,...,M_N$. Hence, from (\ref{eq2})
\begin{align}\label{prob_KL}
\Pi(\mathcal{A}_{1N})&\gtrsim\Pi\left(\left\{\bgamma: \mathbb{E}_{\mathcal{U}}\mathbb{E}_{\mathcal{\bX}}\left[|| \tilde{\bX}_{\bPhi,N}\bB(\bgamma-\bgamma^*)- \tilde{\bX}_{\bPhi,N}\bet^*||_2^2\right]\lesssim M_N\theta_N^2\right\}\right)\nonumber\\
&\geq\Pi\left(\left\{\bgamma: \mathbb{E}_{\mathcal{U}}\mathbb{E}_{\mathcal{\bX}}\left[|| \tilde{\bX}_{\bPhi,N}\bB(\bgamma-\bgamma^*)||_2^2\right] + \mathbb{E}_{\mathcal{U}}\mathbb{E}_{\mathcal{\bX}}\left[||\tilde{\bX}_{\bPhi,N}\bet^*||_2^2\right]\lesssim M_N\theta_N^2/2\right\}\right),
\end{align}
where we use $||\ba-\bb||_2^2\leq 2(||\ba||_2^2+||\bb||_2^2)$, for all $\ba,\bb$. Let $\bB_j(\bu_n)=(B_{j1}(\bu_n),...,B_{jH_N}(\bu_n))^{\T}$, for $n=1,...,N$ and $j=1,...,\tilde{P}$. By Assumption~(E),
\begin{align*}
 \mathbb{E}_{\mathcal{U}}\mathbb{E}_{\mathcal{\bX}}\left[|| \tilde{\bX}_{\bPhi,N}\bB(\bgamma-\bgamma^*)||_2^2\right] &\asymp \kappa_NE_{\mathcal{U}}\mathbb{E}_{\mathcal{\bX}}\left[||\tilde{\bX}_{N}\bB(\bgamma-\bgamma^*)||_2^2\right] \\
 &\quad = \kappa_N(\bgamma-\bgamma^*)^{\T}\mathbb{E}_{\mathcal{U}}\mathbb{E}_{\mathcal{X}}\left[\bB^{\T}\tilde{\bX}_{N}^{\T}\tilde{\bX}_{N}\bB\right](\bgamma-\bgamma^*).
\end{align*}

Recalling that $\bB^{\T}\tilde{\bX}_{N}^{\T}\tilde{\bX}_{N}\bB$ is a $H_N\tilde{P}\times H_N\tilde{P}$ matrix with the $(j,j')$-th block given by $\sum_{n=1}^N\tilde{x}_j(\bu_n)\bB_j(\bu_n)\bB_{j'}(\bu_n)^{\T}\tilde{x}_{j'}(\bu_n)$, we obtain
\begin{align*}
\mathbb{E}_{\mathcal{U}}E_{\mathcal{X}}\left[\sum_{n=1}^N\tilde{x}_j(\bu_n)\bB_j(\bu_n)\bB_{j'}(\bu_n)^{\T}\tilde{x}_{j'}(\bu_n)\right] &\asymp \mathbb{E}_{\mathcal{U}}\left[\sum_{n=1}^N\bB_j(\bu_n)\bB_{j'}(\bu_n)^{\T}\right]\\
&\quad =N\mathbb{E}_{\mathcal{U}}\left[\bB_j(\bu_1)\bB_{j'}(\bu_1)^{\T}\right],
\end{align*}
where the last equation follows since $\bu_1,...,\bu_N$ are i.i.d.. Hence,
\begin{align}\label{eq:use}
\mathbb{E}_{\mathcal{U}}\mathbb{E}_{\mathcal{X}}\left[|| \tilde{\bX}_{\bPhi,N}\bB(\bgamma-\bgamma^*)||_2^2\right]\asymp N\kappa_N\mathbb{E}_{\mathcal{U}}\left[||\bB(\bu_1)(\bgamma-\bgamma^*)||_2^2\right]\asymp N\kappa_N||\bgamma-\bgamma^*||_2^2/H_N,
\end{align}
where $\bB(\bu)=[\bB_1(\bu):\cdots:\bB_{\tilde{P}}(\bu)]^{\T}$. The last expression follows from Lemma A.1 of \cite{huang2004polynomial}. From Assumption~(E) again,
\begin{align}\label{eq:refer}
\mathbb{E}_{\mathcal{U}}\mathbb{E}_{\mathcal{X}}\left[||\tilde{\bX}_{\bPhi,N}\bet^*||_2^2\right] &\asymp\kappa_N\mathbb{E}_{\mathcal{U}}\mathbb{E}_{\mathcal{X}}\left[||\tilde{\bX}_{N}\bet^*||_2^2\right]=\kappa_N\mathbb{E}_{\mathcal{U}}\mathbb{E}_{\mathcal{X}}\left[\sum_{n=1}^N\sum_{j=1}^{\tilde{P}}\tilde{x}_j(\bu_n)^2\eta_{j}^*(\bu_n)^2\right] \nonumber\\ &\asymp\kappa_N\mathbb{E}_{\mathcal{U}}\left[\sum_{n=1}^N\sum_{j=1}^{\tilde{P}}\eta_{j}^*(\bu_n)^2\right]\lesssim N\kappa_NH_N^{-2\xi},
\end{align}
where the last inequality follows from Assumption~(A). From (\ref{prob_KL}),
\begin{align*}
\Pi(\mathcal{A}_{1N}) &\gtrsim\Pi\left(\bgamma: N\kappa_N||\bgamma-\bgamma^*||_2^2/H_N+N\kappa_NH_N^{-2\xi}\lesssim M_N\theta_N^2/2\right) \\ 
&\quad \gtrsim \Pi\left(\bgamma: N\kappa_N||\bgamma-\bgamma^*||_2^2\leq M_NH_N\theta_N^2\right),
\end{align*}
where the last step follows from Assumptions (B) and (E). Using the fact that $\int_{a}^b\exp(-x^2/2)dx\geq \exp(-(a^2+b^2)/2)(b-a)$, we obtain
\begin{align*}
&\Pi\left(\bgamma: N\kappa_N||\bgamma-\bgamma^*||_2^2\leq M_NH_N\theta_N^2\right) \geq \prod_{h,j=1}^{H_N,\tilde{P}}\Pi(|\gamma_{jh}-\gamma_{jh}^*|\leq \theta_N/\sqrt{\tilde{P}})\\
&\qquad\qquad\qquad \geq \exp(-||\bgamma^*||_2^2-\theta_N^2H_N )(2\theta_N/\sqrt{\tilde{P}})^{H_N\tilde{P}} \gtrsim \exp(-M_N\theta_N^2C_2),
\end{align*}
for any $C_2>0$, where the first inequality follows from Assumption~(E) and the last inequality follows from $H_NP\log(\sqrt{\tilde{P}}/2\theta_N)\prec M_N\theta_N^2$ (since $M_N\theta_N^2\asymp M_N^{d/(d+2\xi)}$ while $H_N\prec M_N^{d/(d+2\xi)}$).
\end{proof}

\section{Proof of Lemma~\ref{lem3}}\label{sec:lem3}
\begin{comment}
\begin{lemma}\label{lem3}
Consider testing the hypothesis $H_0:\gamma=\gamma^*$ vs. $H_1:\gamma\in\mathcal{B}_N^c$, where $\mathcal{B}_N=\{\gamma:||\gamma-\gamma^*||_2\leq C_{2w}\theta_NH_N^{1/2}\}$, for a constant $C_{2w}>0$. Then there exists a sequence of test functions $\zeta_N$  such that
\begin{align}
E^*(\zeta_N)\lesssim \exp(-M_N\theta_N^2),\:\:\:\sup\limits_{\gamma\in\mathcal{B}_N^c}E_{\gamma}(1-\zeta_N)\lesssim \exp(-M_N\theta_N^2),
\end{align}
where $E_{\gamma}$ and $E^*$ denote the expectations under the distributions $f(\cdot|\gamma)$ and $f^*(\cdot|\gamma^*)$, respectively.
\end{lemma}
\end{comment}
\begin{proof}
Denote $\tilde{\bX}_{\bPhi,\bB,N}=\tilde{\bX}_{\bPhi,N}\bB$, $\widehat{\bgamma}=(\tilde{\bX}_{\bPhi,\bB,N}^{\T}\tilde{\bX}_{\bPhi,\bB,N})^{-1}\tilde{\bX}_{\bPhi,\bB,N}^{\T}\by_{\bPhi,N}$ and a sequence of random variables $\zeta_N=I(||\tilde{\bX}_{\bPhi,\bB,N}\widehat{\bgamma}-\tilde{\bX}_{\bPhi,\bB,N}\bgamma^*||_2 \gtrsim\theta_NM_N^{1/2})$. Then,
\begin{align*}
\mathbb{E}^*(\zeta_N)&=P^*(||\tilde{\bX}_{\bPhi,\bB,N}\widehat{\bgamma}-\tilde{\bX}_{\bPhi,\bB,N}\bgamma^*||_2 \gtrsim\theta_NM_N^{1/2})\nonumber\\
&\qquad =P^*(||\bP_{\tilde{\bX}_{\bPhi,\bB,N}} \tilde{\bX}_{\bPhi,N}\bet^*+\bP_{\tilde{\bX}_{\bPhi,B,N}}\bepsilon||_2^2\gtrsim \theta_N^2M_N)\nonumber\\
&\qquad \qquad \leq P^*(||\bP_{\tilde{\bX}_{\bPhi,\bB,N}} \tilde{\bX}_{\bPhi,N}\bet^*||_2^2+||P_{\tilde{\bX}_{\bPhi,\bB,N}}\bepsilon||_2^2\gtrsim \theta_N^2M_N),
%&\leq P^*(||\Phi X_{D,N}\eta^*||_2^2+||(\tilde{X}_{D,B,N}'\tilde{X}_{D,B,N})^{-1}\tilde{X}_{D,B,N}'\epsilon||_2^2\gtrsim \theta_N^2H_N).
\end{align*}
where $\bP_{\tilde{\bX}_{\bPhi,\bB,N}}$ denotes the projection matrix corresponding to the matrix $\tilde{\bX}_{\bPhi,\bB,N}$. Note that
\begin{align*}
||\bP_{\tilde{\bX}_{\bPhi,\bB,N}}\tilde{\bX}_{\bPhi,N}\bet^*||_2^2
\leq\bet^{*\T}\tilde{\bX}_{\bPhi,N}^{\T}\bP_{\tilde{\bX}_{\bPhi,\bB,N}} \tilde{\bX}_{\bPhi,N}\bet^*
\leq ||\tilde{\bX}_{\bPhi,N}\bet^*||_2^2.
\end{align*}
%By Lemma~\ref{lem1} and Assumption (D), 
%$e_{min}(\tilde{X}_{D,B,N}'\tilde{X}_{D,B,N})\asymp 1$.
We then refer to equation (\ref{eq:refer}) to see that $E_{\mathcal{U}}E_X||\tilde{\bX}_{\bPhi,N}\bet^*||_2^2\lesssim N\kappa_N M_N^{-2\xi/(2\xi+d)}\prec M_N\theta_N^2$. The above two facts together conclude that
\begin{align*}
\mathbb{E}_{\mathcal{U}}\mathbb{E}_X[||\bP_{\tilde{\bX}_{\bPhi,\bB,N}} \tilde{\bX}_{\bPhi,N}\bet^*||_2^2]\lesssim N\kappa_N M_N^{-2\xi/(2\xi+d)}\prec M_N\theta_N^2.
\end{align*}
$\mathbb{E}^*(\zeta_N)\lesssim P^*(||\bP_{\tilde{\bX}_{\bPhi,\bB,N}}\bepsilon||_2^2\gtrsim \theta_N^2M_N)= P^*(\bepsilon^{\T}\bP_{\tilde{\bX}_{\bPhi,\bB,N}}\bepsilon\gtrsim \theta_N^2M_N)$.

Note that under $P^*$, $\bepsilon\sim N(0,\bPhi\bPhi^{\T})$, and,  $e_{max}(\bPhi\bPhi^{\T})\asymp 1$ (by Lemma~\ref{lem1}). Also note that Lemma 1 of \cite{laurent2000adaptive} can be simplified to write $P^*(\chi_{p^*}^2>x)\leq \exp(-x/4)$, for $x\geq 8p^*$. Further, $\bepsilon^{\T}\bP_{\tilde{\bX}_{\bPhi,\bB,N}}\bepsilon$ follows a $\chi^2$ distribution with degree of freedom less than equal to $H_N\tilde{P}\prec M_N\theta_N^2=M_N^{d/(d+2\xi)}$. Using all the above facts, we conclude that $E^*(\zeta_N)\lesssim \exp(-M_N\theta_N^2)$.

Next, for $\bgamma\in\mathcal{B}_N^c$, we show that $\mathbb{E}_{\mathcal{U}}\mathbb{E}_{\mathcal{X}}||\tilde{\bX}_{\bPhi,\bB,N}\bgamma-\tilde{\bX}_{\bPhi,\bB,N}\bgamma^*||_2^2\gtrsim M_N\theta_N^2$. To see this, note that
\begin{align*}
& \mathbb{E}_{\mathcal{U}}E_{\mathcal{X}}||\tilde{\bX}_{\bPhi,\bB,N}\bgamma-\tilde{\bX}_{\bPhi,\bB,N}\bgamma^*||_2^2 = \mathbb{E}_{\mathcal{U}}\mathbb{E}_{\mathcal{X}}\left[(\bgamma-\bgamma^*)^{\T}\tilde{\bX}_{\bPhi,\bB,N}^{\T}\tilde{\bX}_{\bPhi,\bB,N}(\bgamma-\bgamma^*)\right]\nonumber\\
&\qquad\asymp \kappa_N \mathbb{E}_{\mathcal{U}}\mathbb{E}_{\mathcal{X}}\left[(\bgamma-\bgamma^*)^{\T}\bB^{\T}\tilde{\bX}_{N}^{\T}\tilde{\bX}_{N}\bB(\bgamma-\bgamma^*)\right]
\asymp N\kappa_N||\bgamma-\bgamma^*||_2^2/H_N\gtrsim M_N\theta_N^2,
\end{align*}
where the second line follows using similar calculations leading to equation (\ref{eq:use}). 

Now, using the fact that 
$||\tilde{\bX}_{\bPhi,\bB,N}\widehat{\bgamma}-\tilde{\bX}_{\bPhi,\bB,N}\bgamma||_2\geq -||\tilde{\bX}_{\bPhi,\bB,N}\widehat{\bgamma}-\tilde{\bX}_{\bPhi,\bB,N}\bgamma^*||_2+||\tilde{\bX}_{\bPhi,\bB,N}\bgamma-\tilde{\bX}_{\bPhi,\bB,N}\bgamma^*||_2$, we obtain
\begin{align*}
\mathbb{E}_{\gamma}(1-\zeta_N) &= P_{\bgamma}(||\tilde{\bX}_{\bPhi,\bB,N}\widehat{\bgamma}-\tilde{\bX}_{\bPhi,\bB,N}\bgamma^*||_2\lesssim \theta_NM_N^{1/2}) \\ 
&= P_{\bgamma}(||\tilde{\bX}_{\bPhi,\bB,N}\widehat{\bgamma}-\tilde{\bX}_{\bPhi,\bB,N}\bgamma||_2\gtrsim \theta_NM_N^{1/2})\nonumber\\
&\leq P_{\gamma}(||\bP_{\tilde{\bX}_{\bPhi,\bB,N}}\bepsilon||_2^2\gtrsim \theta_N^2M_N)\lesssim \exp(-M_N\theta_N^2),
\end{align*}
where the last inequality follows from simplifying the conclusion for Lemma 1 of \cite{laurent2000adaptive} (as is done before) and the fact that under $P_{\bgamma}$, $\bepsilon\sim N({\boldsymbol 0},\bI)$. 
\end{proof}

\section{Proof of Theorem~\ref{thm_main}}\label{sec:thm_main}
%Suppose that assumptions (A1)-(A4) holds. Let $w(s)=(w_1(s),...,w_p(s))'$ and $w_0(s)=(w_{0,1}(s),...,w_{0,p}(s))'$. Then for all $w_0(s)\in\mathcal{H}_{\xi}$ and with $\epsilon_n^2\sim m_n^{-1/(1+\xi)}$, we have
%\begin{align}\label{eq3}
%E_0\Pi(||w-w_0||_{L_2}\geq C_w\epsilon_n|\Phi y)\rightarrow 0,\:\:\mbox{as}\:\:n\rightarrow\infty,
%\end{align}
%for some constant $C_w>0$.
%\end{theorem}
\begin{proof}
Note that,
\begin{align*}
||\bw-\bw^*||_{2} &\leq ||\bw-\tilde{\bw}^*+\tilde{\bw}^*-\bw^*||_{2}\leq ||\bw-\tilde{\bw}^*||_{2}+||\tilde{\bw}^*-\bw^*||_{2}=||\bw-\tilde{\bw}^*||_{2}+||\bet^*||_{2}\nonumber\\
&\quad \lesssim ||\bw-\tilde{\bw}^*||_{2}+P^{1/2}H_N^{-\xi}\asymp ||\bgamma-\bgamma^*||_2H_N^{-1/2}+P^{1/2}H_N^{-\xi} \\
&\quad\quad \asymp ||\bgamma-\bgamma^*||_2H_N^{-1/2}+P^{1/2}M_N^{-\xi/(2\xi+d)},
\end{align*}
where $\tilde{\bw}^*(\bu)=(\sum_{h=1}^{H_N}B_{1h}(\bu)\gamma_{1h}^*,\ldots,\sum_{h=1}^{H_N}B_{\tilde{P}h}(\bu)\gamma_{\tilde{P}h}^*)^{\T}$, and the first inequality in the second line follows from the property of B-splines \citep{huang2004polynomial}. The second expression in the second line follows from Lemma A.1 of \cite{huang2004polynomial}. Using the fact that $\tilde{P}^{1/2}M_N^{-\xi/(2\xi+d)}=O(\theta_N)$,
we have $\left\{\bw: ||\bw-\bw^*||_{2}\geq \tilde{C}\theta_N\right\}\subset\left\{\bgamma: ||\bgamma-\bgamma^*||_2H_N^{-1/2}\geq C_{2w}\theta_N\right\}$, for some constant $C_{2w}>0$.

Denote $\mathcal{B}_N = \left\{\bgamma: ||\bgamma-\bgamma^*||_2H_N^{-1/2}\leq C_{2w}\theta_N\right\}$. To prove the theorem, it is enough to establish
\begin{align}\label{eq4}
\mathbb{E}^*\Pi(||\bgamma-\bgamma^*||_{2}H_N^{-1/2}\geq C_{2w}\theta_N| \by_{\bPhi,N},\tilde{\bX}_{\bPhi,N})\rightarrow 0,\:\:\mbox{as}\:\:N\rightarrow\infty,
\end{align}
Note that,
\begin{align}
& \mathbb{E}^*[\Pi(\mathcal{B}_N^c\given \by_{\bPhi,N},\tilde{\bX}_{\bPhi,N})]\leq \mathbb{E}^*\zeta_N+ \mathbb{E}^*[\Pi(\mathcal{B}_N^c\given \by_{\bPhi,N},\tilde{\bX}_{\bPhi,N})(1-\zeta_N)1_{y_N\in\mathcal{A}_N^c}] + P^*(\mathcal{A}_N)\nonumber\\
&\; = \mathbb{E}^*[\zeta_N] + \mathbb{E}^*\left[1_{\by_N\in\mathcal{A}_N^c}\frac{\left\{(1-\zeta_N)\int_{\mathcal{B}_N^c}\{f(\by_{\bPhi,N}|\bgamma)/f^*( \by_{\bPhi,N}|\bgamma^*)\}\pi_N(\bgamma)d\bgamma\right\}}{\left\{\int\{f(\by_{\bPhi,N}|\bgamma)/f^*( \by_{\bPhi,N}|\bgamma^*)\}\pi_N(\bgamma)d\bgamma\right\}}\right] + P^*(\mathcal{A}_N),
\end{align}
where $\mathcal{A}_N$ is a set defined in the statement of Lemma~\ref{lem2} and $\zeta_N$ can be regarded as a sequence of random variables as defined in %$H_0:\gamma=\gamma^*$ vs. $H_1:\gamma\in\mathcal{B}_N^c$%, as defined in
Lemma~\ref{lem3}.
By Lemma~\ref{lem2}, $P^*(\mathcal{A}_N)\rightarrow 0$, as $N, M_N\rightarrow\infty$. Also, by Lemma~\ref{lem3}, $\mathbb{E}^*\zeta_N\rightarrow 0$, as $N, M_N\rightarrow\infty$. To show (\ref{eq4}), it remains to prove that $$
\frac{\mathbb{E}^*\left[1_{\by_N\in\mathcal{A}_N^c}\int_{\mathcal{B}_N^c}\{f( \by_{\bPhi,N}|\bgamma)/f^*( \by_{\bPhi,N}|\bgamma^*)\}\pi_N(\bgamma)d\bgamma\right]}{\left[\int\{f( \by_{\bPhi,N}|\bgamma)/f^*( \by_{\bPhi,N}|\bgamma^*)\}\pi_N(\bgamma)d\bgamma\right]}
%\mathbb{E}^*\left[1_{y_N\in\mathcal{A}_N^c}\int_{\mathcal{B}_N^c}\{f( y_{\Phi,N}|\gamma)/f^*( y_{\Phi,N}|\gamma^*)\}\pi_N(\gamma)d\gamma\right]/\left[\int\{f( y_{\Phi,N}|\gamma)/f^*( y_{\Phi,N}|\gamma^*)\}\pi_N(\gamma)d\gamma\right]
\rightarrow 0\; \mbox{ as $N,M_N\rightarrow\infty$}.$$ To this end, we have
\begin{align*}
\mathbb{E}^*\left[1_{\by_N\in\mathcal{A}_N^c}\int_{\mathcal{B}_N^c}\{f( \by_{\bPhi,N}|\bgamma)/f^*( \by_{\bPhi,N}|\bgamma^*)\}\pi_N(\bgamma)d\bgamma\right] &\leq \sup\limits_{\bgamma\in\mathcal{B}_N^c}\mathbb{E}_{\bgamma}(1-\zeta_N)\Pi(\mathcal{A}_N^c) \\
&\quad \leq \exp(-C_{2w}M_N\theta_N^2),
\end{align*}
where $\Pi(\mathcal{A}_N^c)$ is the prior probability of the set $\mathcal{A}_N^c$. The denominator 
$$\int\{f( \by_{\bPhi,N}|\bgamma)/f^*(\by_{\bPhi,N}|\bgamma^*)\}\pi(\bgamma)d\bgamma\geq \exp(-C_1M_N\theta_N^2)$$ on $\mathcal{A}_N$, where $C_1$ is chosen so that $C_1<C_{2w}$. Thus, $\mathbb{E}^*\Pi(\mathcal{B}_N^c\given \by_{\bPhi,N},\tilde{\bX}_{\bPhi,N})1_{\by_N\in\mathcal{A}_N^c}\leq \exp(-(C_{2w}-C_1)M_N\theta_N^2)\rightarrow 0$, as $N,M_N\rightarrow\infty$.
\end{proof}

\section{Proof of Theorem~\ref{thm_main2}}\label{sec:thm_main2}
\begin{proof} For densities $f_u$ and $f^*$, we have
\begin{align*}
h(f_u,f^*) &= 1-\exp\left\{-\left(\sum_{j=1}^{\tilde{P}}\tilde{x}_j(\bu_0)w_j(\bu_0)-\sum_{j=1}^{\tilde{P}}\tilde{x}_j(\bu_0)w_j^*(\bu_0)\right)^2/8\right\}\nonumber\\
&\leq 1-\exp\left\{-\tilde{P}\sum_{j=1}^{\tilde{P}}\left(w_j(\bu_0)-w_j^*(\bu_0)\right)^2/8\right\} \\
&\leq 1-\exp\left\{-\tilde{P}||\bw(\bu_0)-\bw^*(\bu_0)||_2^2/8\right\}
\end{align*}
Then, $\mathbb{E}_{\mathcal{U}}[h(f_u,f^*)]\leq 1-\exp\left(-\tilde{P}||\bw-\bw^*||_2^2/8\right)$, by Jensen's inequality.  Further,
\begin{align*}
\mathbb{E}^* \mathbb{E} \mathbb{E}_{\mathcal{U}}[h(f_u,f^*)|\tilde{\bX}_{\bPhi,N},\by_{\bPhi,N}]=
\left\{1-\exp\left(-\tilde{P}\tilde{C}^2\theta_N^2/8\right)\right\}+
2\Pi_N(||\bw-\bw^*||_2\geq \tilde{C}\theta_N),
%K(\widehat{f},f^*) &=E_{s}E_{x}[\{\sum_{p=1}^Px_p(s)(\widehat{\beta}_p(s)-\beta_p^*(s)\}^2]
%\leq E_{s}E_{x}[\{\sum_{p=1}^Px_p(s)(\widehat{\beta}_p(s)-\beta_p^*(s)\}^2]\nonumber\\
%&\leq PE_s[\sum_{p=1}^P(\widehat{\beta}_p(s)-\beta_p^*(s))^2]\leq P||\widehat{\beta}-\beta^*||_{2}^2\leq PE[||\beta-\beta^*||_2^2|\tilde{y}_N,\tilde{X}_N]\nonumber\\
%&\leq \tilde{C}^2\theta_N^2\Pi(||\beta-\beta^*||_2\leq \tilde{C}\theta_N)+\\
%&\leq 2E_{\mathcal{S}}E_{\mathcal{X}}[||X_{D,N}B(\widehat{\gamma}_N-\gamma^*)||^2]+2E_{\mathcal{S}}E_{\mathcal{X}}[||X_{D,N}\eta^*||^2]
\end{align*}
which implies
\begin{align*}
\mathbb{E}^*\mathbb{E} \mathbb{E}_{\mathcal{U}}[h(f_u,f^*)]
\leq \left\{1-\exp\left(-\tilde{P}\tilde{C}^2\theta_N^2/8\right)\right\}+
2\mathbb{E}^*\Pi_N(||\bw-\bw^*||_2\geq \tilde{C}\theta_N)\rightarrow 0
\end{align*}
as $N,M_N\rightarrow\infty$, where the last expression followed by the conclusion of Theorem~\ref{thm_main} and the fact that $\theta_N\rightarrow 0$ as $N,M_N\rightarrow\infty$.
\end{proof}

\bibliography{reference_biom}
\end{document}